\documentclass{article}

\usepackage[utf8]{inputenc} %
\usepackage[T1]{fontenc}    %
\usepackage{url}            %
\usepackage{amsfonts}       %
\usepackage{nicefrac}       %
\usepackage{xcolor}         %

\usepackage{xspace}
\usepackage{bbm}
\usepackage{pifont}
\usepackage{caption}
\usepackage{subcaption}
\usepackage{wrapfig}
\usepackage{enumitem}
\usepackage{multirow}
\usepackage{comment}

\usepackage{makecell}

\usepackage{xspace}
\newcommand*{\ie}{{\it i.e.},\@\xspace} %
\newcommand{\g}{g \cdot}
\newcommand{\ginv}{g^{-1} \cdot}

\usepackage{microtype}
\usepackage{graphicx}
\usepackage{booktabs} %

\usepackage{hyperref}

\usepackage[accepted]{icml2024}

\usepackage{amsmath}
\usepackage{amssymb}
\usepackage{mathtools}
\usepackage{amsthm}

\usepackage[capitalize,noabbrev]{cleveref}

\usepackage{miller}

\theoremstyle{plain}
\newtheorem{theorem}{Theorem}[section]

\theoremstyle{definition}

\theoremstyle{remark}

\usepackage[disable,textsize=tiny]{todonotes}

\icmltitlerunning{FlowMM: Generating Materials with Riemannian Flow Matching}

\everypar{\looseness=-1}
\begin{document}

\twocolumn[
\icmltitle{FlowMM: Generating Materials with Riemannian Flow Matching}

\icmlsetsymbol{equal}{*}

\begin{icmlauthorlist}
\icmlauthor{Benjamin Kurt Miller}{ams,meta}
\icmlauthor{Ricky T. Q. Chen}{meta}
\icmlauthor{Anuroop Sriram}{meta}
\icmlauthor{Brandon M. Wood}{meta}
\end{icmlauthorlist}

\icmlaffiliation{ams}{University of Amsterdam}
\icmlaffiliation{meta}{FAIR, Meta AI}

\icmlcorrespondingauthor{BKM}{b.k.miller@uva.nl}
\icmlcorrespondingauthor{BW}{bmwood@meta.com}

\icmlkeywords{Machine Learning, ICML}

\vskip 0.3in
]

\printAffiliationsAndNotice{Research done while BKM was an intern at FAIR, Meta AI.}  %

\begin{abstract}
Crystalline materials are a fundamental component in next-generation technologies, yet modeling their distribution presents unique computational challenges. Of the plausible arrangements of atoms in a periodic lattice only a vanishingly small percentage are thermodynamically stable, which is a key indicator of the materials that can be experimentally realized. Two fundamental tasks in this area are to (a) predict the stable crystal structure of a known composition of elements and (b) propose novel compositions along with their stable structures. We present FlowMM, a pair of generative models that achieve state-of-the-art performance on both tasks while being more efficient and more flexible than competing methods. We generalize Riemannian Flow Matching to suit the symmetries inherent to crystals: translation, rotation, permutation, and periodic boundary conditions. Our framework enables the freedom to choose the flow base distributions, drastically simplifying the problem of learning crystal structures compared with diffusion models. In addition to standard benchmarks, we validate FlowMM's generated structures with quantum chemistry calculations, demonstrating that it is $\sim$3x more efficient, in terms of integration steps, at finding stable materials compared to previous open methods.
\end{abstract}

\begin{figure}[th]
    \vskip 0.0in
    \begin{center}
    \centerline{\includegraphics[width=\columnwidth]{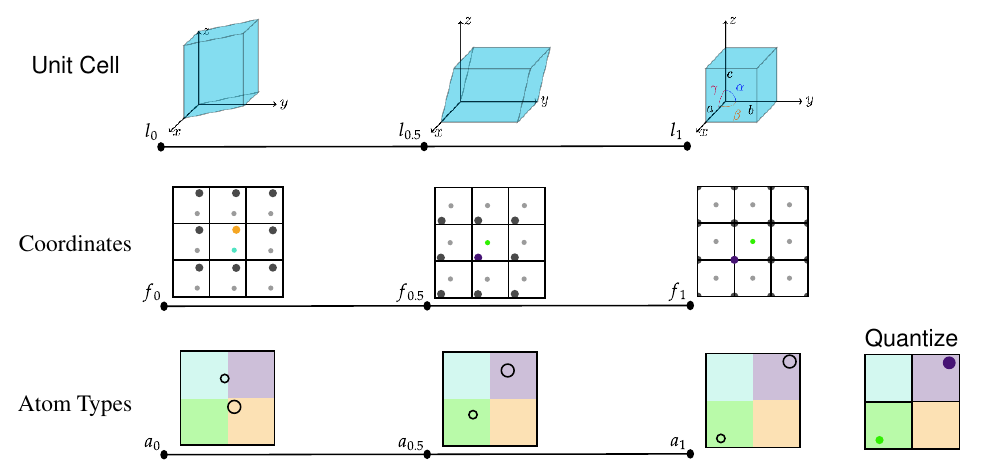}}
    \caption{%
        A conceptual representation of the evolution from base distribution to target distribution, according the vector field learned by our model. 
        We model a joint distribution over lattice parameters, periodic fractional coordinates, and a binary representation of atom type. 
        The highlights include symmetry-aware geodesic paths and a base distribution that directly produces plausible lattices. Note that coordinates and atom types are merely depicted in 2d for clarity.
    }
    \label{fig:conceptual}
    \end{center}
    \vskip -0.3in
\end{figure}

\vspace{-1em}

\section{Introduction}

Materials discovery has played a critical role in key technological developments across history \cite{appl1982haber}. The huge number of plausible materials offers the potential to advance key areas such as energy storage \cite{ling2022review}, carbon capture \cite{hu2023machine}, and microprocessing \cite{choubisa2023interpretable}. With the promise comes a challenge: 
The search space of crystal structures is combinatorial in the number of atoms and elements under consideration, resulting in an intractable number of potential structures. Of these, only a vanishing small percentage will be experimentally realizable.
These factors have motivated computational exploration of the materials design space, which has accelerated discovery campaigns, but typically relies on random structure search \cite{potyrailo2011combinatorial} and expensive quantum mechanical computations. Generative modeling has shown promise for navigating large design spaces for scientific research, but the application to materials is still relatively novel.

The specific task we will focus on in this paper is the generation of periodic crystals. We're interested in two cases: \emph{Crystal Structure Prediction} (CSP), which we define as the setting where the user provides the number of constituent atoms and their elemental types, i.e. their \emph{composition}, and \emph{De Novo Generation} (DNG) when the generative model is tasked to produce the composition alongside its crystal structure. 
An additional complexity is that we want to generate crystals that are \emph{stable}. Naively, stability is determined by a thermodynamic competition between a structure and competing alternatives. The known stable structures define a \emph{convex hull} of stable compositions over the energy landscape. 
We use the heuristic that the stability of a material is determined by its energetic distance to the convex hull, denoted $E_{hull}$. 
Structures with $E_{hull} \leq 0.0$ eV/atom are considered stable. 
Otherwise, the structure is unfavorable and will decompose into a stable neighboring composition.

Modeling crystals is challenging because it requires fitting distributions jointly over several variables, each with different and interdependent structure. 
The atom types are discrete, while the \emph{unit cell}, i.e. the repeating fundamental volume element, and the atomic positions are continuous.
Side lengths of the unit cell are strictly positive, the angles are bounded, and the atomic coordinates are periodic. 
The number of atoms varies, but the unit cell dimensionality is fixed.

While diffusion models found some success in generating stable materials they have fundamental limitations, making them suboptimal for the problem. 
In graph neural network implementations, each variable has required a different diffusion framework to fit its structure \cite{jiao2023crystal, zeni2023mattergen}. 
Atomic coordinates are modeled using score matching \cite{song2020score}, enabling a wrapped normal transition distribution and a uniform asymptotic distribution,
atomic types utilize discrete diffusion \cite{austin2021structured}, and the unit cell is modeled with denoising diffusion \cite{sohl2015deep, ho2020denoising}.
Both of these approaches must choose a complex representation for the lattice in order to fit within the diffusion framework since their limiting distribution must be normal, yet still represent a ``realistic material'' while doing diffusion. Furthermore \citet{jiao2023crystal}, must perform an ad hoc Monte Carlo estimate of a time evolution scaling factor (top of page 6 in their paper), and they do not simulate the forward trajectory in training. This simulation is a slow, but necessary step, in rigorous diffusion models on manifolds \cite{huang2022riemannian}.

\paragraph{Our contribution}
We introduce \emph{FlowMM}, a pair of generative models for CSP and DNG that each jointly estimate symmetric distributions over fractional atomic coordinates and the unit cell (along with atomic types for DNG) in a single framework based on Riemannian Flow Matching \cite{lipman2022flow, chen2023riemannian}. We train a Continuous Normalizing Flow \cite{chen2018neural} with a finite time evolution and produce high-quality samples, as measured by standard metrics and thermodynamic stability, with significantly fewer integration steps than diffusion models.

First, we generalize the Riemannian flow matching framework to estimate a point cloud density that is invariant to translation with periodic boundary conditions, a novel achievement for continuous normalizing flows, by proposing a new objective in Section~\ref{sec:rfm4mat}. With this step, it becomes possible to enforce isometric invariances inherent to the geometry of crystals as an inductive bias in the generative model. Second, after selecting a rotation invariant representation, we choose a natural base distribution that samples plausible unit cells by design. We find that this drastically simplifies fitting the lattice compared with diffusion models, which are forced to take an unnatural base distribution due to inherent limitations in their framework. Third, for DNG, we choose a binary representation \cite{chen2022analog} for the atom types that drastically reduces the dimensionality compared with the simplex (one-hot). Our representation is $\lceil \log_2(100) \rceil = 7$ dimensions per atom, while the simplex requires 100 dimensions per atom. (Note that $\lceil \cdot \rceil$ denotes the ceiling function.) We attribute the accuracy of FlowMM in predicting the number of unique elements per unit cell to this design choice, something other models struggle with. Finally, we compare our method to diffusion model baselines with extensive experiments on two realistic datasets and two simplified unit tests. In addition to competitive or state-of-the-art performance on standard metrics, we take the additional step to estimate thermodynamic stability of generated structures by performing expensive quantum chemistry calculations. We find that FlowMM is able to generate materials with comparable stability to these other methods, while being \emph{significantly} faster at inference time.

We made our code publicly available on \texttt{GitHub}\footnote{\href{https://github.com/facebookresearch/flowmm}{https://github.com/facebookresearch/flowmm}}.

\paragraph{Related work} 
The earliest approaches for both CSP and DNG new materials rely on proposing a large number of possible candidate materials and screening them with high-throughput quantum mechanical \cite{kohn1965self} calculations to estimate the energy and find stable materials. Those candidate materials are proposed using simple replacement rules \cite{wang2021predicting} or accelerated by genetic algorithms \cite{glass2006uspex, pickard2011ab}. This search can be accelerated using machine learned models to predict energy \cite{schmidt2022large, merchant2023scaling}.

Various generative models have been designed to accelerate material discovery by directly generating materials, avoiding expensive brute force search \cite{court20203, yang2021crystal, nouira2018crystalgan}. Diffusion models have been widely used for this task. Initially without diffusing the lattice but predicting it with a variational autoencoder \cite{xie2021crystal}; later by jointly diffusing the positions, lattice structure and atom types ~\cite{jiao2023crystal,yang2023scalable,zeni2023mattergen}. We only compare to open models with verifiable results at the time of submission of this paper to ICML 2024 in January. More recently, other models have used space groups as additional inductive bias~\cite{ai4science2023crystal, jiao2024space, cao2024space}. Other approaches include using large language models~\cite{flam2023language, gruver2024fine}, and with normalizing flows \citet{wirnsberger2022normalizing}.

\section{Preliminaries}
\label{sec:preliminaries}

We are concerned with fitting probability distributions over crystal lattices, which are collections of atoms periodically arranged to fill space in a repeated pattern. One way to construct a three-dimensional crystal is by tiling a parallelepiped, or \emph{unit cell}, such that the entirety of space is covered. The unit cell contains an arrangement of atoms, i.e., a labeled point cloud, which produces the crystal lattice under the tiling. The following is the exposition towards our generative model. As background, we recommend a primer on the symmetries of crystals \citep{adams2023representing}.

\subsection{Representing crystals and their symmetries}
\label{sec:representing_crystals}

\paragraph{Representation of a crystal}

We label the particles in a crystal $\cboldt \in \Tilde{\Ccal}$ with $n \in \Nbb$ sites (atoms) by column in matrix 
$\abold \coloneqq \left[a^{1}, \ldots, a^{n}\right] \in \Acal
$
where each atomic number maps to a unique $h$-dimensional categorical vector. The unit cell is a matrix of Cartesian column vectors $\lboldt \coloneqq \left[\Tilde{l}^{1}, \Tilde{l}^{2}, \Tilde{l}^{3}\right] \in \Tilde{\Lcal} = \Rbb^{3 \times 3}$. The position of each particle can be represented using a matrix $\xbold \coloneqq \left[x^{1}, \ldots, x^{n}\right] \in \Xcal=\Rbb^{3 \times n}$ with Cartesian coordinates in the rows and particles in the columns, but we will choose an alternative fractional representation that leverages the periodic boundary of $\lboldt$. Positions are equivalently $\xbold = \lboldt \fbold$ where $\fbold \coloneqq \lboldt^{-1} \xbold = \left[f^{1}, \ldots, f^{n}\right] \in \Fcal = [0,1)^{3 \times n}$. The volume of a unit cell $\Vol(\lboldt) \coloneqq \lvert \det \lboldt \rvert$ must be nonzero, implying that $\lboldt$ is invertible.
Then, the crystal is
\begin{align*}
    \left\{
    \left(\abold', \fbold' \right)
    \mmid 
    \abold' = \abold,
    \fbold' = \fbold + \lboldt k \boldsymbol{1}, \forall k \in \Zbb^{3 \times 1}
    \right\}
\end{align*}
where elements of $k$ denote integer translations of the lattice and $\boldsymbol{1}$ is a $1 \times n$ matrix of ones to emulate ``broadcasting.''

This representation is not unique since there is freedom to choose an alternative unit cell and corresponding fractional coordinates producing the same crystal lattice, e.g. by doubling the volume or skewing the unit cell. Among minimum volume unit cells, we choose $\lboldt$ to be the unique one determined by Niggli reduction \cite{grosse2004numerically}.

\paragraph{Equivariance}
A function $f \colon \Scal \to \Scal'$ is \emph{$G$-equivariant} when for any $s \in \Scal$ and any $g \in G$ we have $f(g \cdot s) = g \cdot f(s)$ where $\cdot$ indicates group action in the relevant space. \emph{$G$-invariant} functions have instead $f(g \cdot s) = f(s)$. We will apply graph neural networks \cite{satorras2021n} with these properties \cite{thomas2018tensor, kondor2018generalization, miller2020relevance, weiler2021coordinate, geiger2022e3nn, liao2023equiformerv2, passaro2023reducing}.

\paragraph{Invariant density}
A density $p \colon \Scal \to \Rbb^{+}$ is \emph{$G$-invariant} when $p$ is $G$-invariant. When a $G$-invariant density $p$ is transformed by an invertible, $G$-equivariant function $f$ the resulting density is $G$-invariant \cite{kohler2020equivariant}.

\paragraph{Symmetries of crystals}
We will now introduce relevant symmetry groups and their action on our crystal representation $\cboldt \coloneqq (\abold, \fbold, \lboldt) \in \Acal \times \Fcal \times \Tilde{\Lcal} \eqqcolon \Tilde{\Ccal}$. The symmetric group $S_n$ on $n$ atoms has $n!$ elements corresponding to all permutations of atomic index. The element $\sigma \in S_n$ acts on a crystal by $\sigma \cdot \cboldt = \left(\left[ a^{\sigma(1)}, \ldots, a^{\sigma(n)} \right], \left[f^{\sigma(1)}, \ldots, f^{\sigma(n)}\right], \lboldt \right)$. The special Euclidean group $\SE(3)$ consists of orientation preserving rigid rotations and translations in Euclidean space. The elements $(Q, \tau) \in \SE(3)$ can be decomposed into rotation $Q \in \SO(3)$, represented by a $3 \times 3$ rotation matrix, and translation $\tau \in [-\tfrac{1}{2}, \tfrac{1}{2}]^{3 \times 1}$. Considering the periodic unit cell, the action on our crystal representation is ${\tau \cdot \cboldt = (\abold, \fbold + \tau \boldsymbol{1} - \lfloor \fbold + \tau \boldsymbol{1} \rfloor, \lboldt)}$, where $\lfloor \cdot \rfloor$ denotes the element wise floor function, i.e., $\fbold$ coordinates ``wrap around.'' The rotation action is $Q \cdot \cbold \coloneqq (\abold, \fbold, Q\lboldt)$.

The true distribution $q(\cboldt)$ has invariances to group operations that we would like our estimated density to inherit:
\begin{alignat}{2}
    q(\sigma \cdot \cboldt) &= q(\sigma \cdot \abold, \sigma \cdot \fbold, \lboldt) = q(\cboldt), &\quad& \forall \sigma \in S_{n} \label{eqn:permutation_invariance} \\
    q(\tau \cdot \cboldt) &= q(\abold, \fbold + \tau \boldsymbol{1} - \lfloor \fbold + \tau \boldsymbol{1} \rfloor, \lboldt)  \nonumber \\ 
    &= q(\cboldt), && \hspace{-2em} \forall \tau \in [-\tfrac{1}{2}, \tfrac{1}{2}]^{3} \label{eqn:translation_invariance} \\
    q(Q \cdot \cboldt) &= q(\abold, \fbold, Q \lboldt) = q(\cboldt), & Q &\in \SO(3) \label{eqn:rotation_invariance}
\end{alignat}
permutation, translation, and rotation invariance, respectively. We address \eqref{eqn:permutation_invariance} and \eqref{eqn:translation_invariance} in Section~\ref{sec:rfm4mat}, but \eqref{eqn:rotation_invariance} here. 

Probability distributions over a $G$-invariant representation are necessarily $G$-invariant. Lattice parameters $\lbold \in \Lcal$ are a rotation invariant representation of the unit cell as a 6-tuple of three side lengths $a, b, c \in \Rbb^{+}$ with units of \ang and three internal angles $\alpha, \beta, \gamma \in [60^\circ, 120^\circ]$ in degrees. (This range for the lattice angles is due to the Niggli reduction.) We therefore propose a rotation invariant alternative crystal representation $\cbold \in \Ccal \coloneqq \Acal \times \Fcal \times \Lcal$ and thereby any distribution $p(\cbold)$ is rotation invariant:
\begin{align*}
    Q \cdot \cbold &= (\abold, \fbold, Q \cdot \lbold) = (\abold, \fbold, \lbold) = \cbold, &   p(Q \cdot \cbold) = q(\cbold).
\end{align*}
$\lbold$ carries all non-orientation information about the unit cell. By composing functions $\Ufrak \colon \lboldt \to \lbold$ and $\Ufrak^{\dagger} \colon \lbold \to \lboldt$, implemented by \citet{ong2013python}, we can reconstruct $\lboldt$ from $\lbold$ up to rotation, i.e. $\Ufrak^{\dagger}( \Ufrak(\lboldt) ) = Q \lboldt$ for some $Q \in \SO(3)$.
Further symmetry information is in Appendix~\ref{appendix:preliminaries}.

\paragraph{Representing atomic types}
The representation of $\abold \in \Acal$ depends on whether the generative model is doing CSP or DNG. For CSP, the atomic types are only conditional information and may be considered a tuple of $n$, $h$-dimensional one-hot vectors. For DNG, the generative model treats $\abold$ as a random variable and learns a distribution over its representation. We choose to apply the binarization method proposed by \citet{chen2022analog} where categorical vector $\abold$ is mapped to a $\{-1, 1\}$-bit representation of length $\lceil \log_2 h \rceil$. The flow then learns to transform a $\lceil \log_2 h \rceil$-dimensional normal distribution to the corresponding bit representation and, at inference time, is discretized by the $\sign \colon \Rbb \to \{-1, 1\}$ function. When $\lceil \log_2 h \rceil \neq \log_2 h$, we end up with ``unused bits'', i.e. we can represent more than $h$ classes. We find that the model is able to learn to ignore these extra atom types in practice. Note that \citet{chen2022analog} suggest using a self-conditioning scheme, but we did not find it necessary.

\subsection{Learning distributions with Flow Matching}
\label{sec:flow_matching}

\paragraph{Riemannian Manifolds (Abridged)}
In order to learn probability distributions over spaces that ``wrap around,'' we must introduce smooth, connected Riemannian manifolds $\Mcal$. They relax the notion of a global coordinate system in lieu of a collection of local coordinate systems that transition seamlessly between one another. Every point $m \in \Mcal$ has an associated \emph{tangent space} $\Tcal_{m} \Mcal$ with an inner product $\lA u, v \rA$ for $u, v \in \Tcal_{m} \Mcal$, enabling the definition of distances, volumes, angles, and minimum length curves (\emph{geodesics}). A fundamental building block for our generative model is the space of time-dependent, smooth vector fields $\Ucal$ living on the manifold. Elements $u_t \in \Ucal$ are maps $u_t: [0, 1] \times \Mcal \to \Tcal \Mcal$ where the first argument $t$ denotes time and $\Tcal \Mcal \coloneqq \cup_{m\in\Mcal} \{m\} \times \Tcal_{m} \Mcal$ denotes the \emph{tangent bundle}, i.e., the disjoint union of each tangent space at every point on the manifold. We learn distributions by estimating functions in $\Ucal$.

The geodesics for any $\Mcal$ that we consider can be written in closed form using the exponential and logarithmic maps. The geodesic connecting $m_0, m_1 \in \Mcal$ at time $t\in\lB0, 1\rB$ is
\begin{equation}
\label{eqn:geodesic}
    m_t \coloneqq \exp_{m_0}( t \log_{m_0}(m_1)),
\end{equation}
where $\exp_\square$ and $\log_\square$ depend on the manifold $\Mcal$.

\paragraph{Probability paths on flat manifolds}
Probability densities on $\Mcal$ \footnote{We consider flat tori and Euclidean space, restricting ourselves to manifolds with flat metrics, i.e. the metric is the identity matrix.} are continuous functions $p \colon \Mcal \to \Rbb^{+}$ where $\int_\Mcal p(m) \, dm = 1$ and $p \in \Pcal$, the space of probability densities on $\Mcal$. A \emph{probability path} $p_t: [0, 1] \to \Pcal$ is a continuous time-dependent curve in probability space linking two densities $p_0, p_1 \in \Pcal$ at endpoints $t=0$, $t=1$. A \emph{flow} $\psi_t \colon [0, 1] \times \Mcal \to \Mcal$ is a time-dependent diffeomorphism defined to be the solution to the ordinary differential equation: $\frac{d}{dt} \psi_{t}(m) = u_t(\psi_t(m))$, subject to initial conditions $\psi_0(m) = m_0$ with $u_t \in \Ucal$. A flow $\psi_t$ is said to generate $p_t$ if it pushes $p_0$ forward to $p_1$ following the time-dependent vector field $u_t$. The path is denoted $p_t = [\psi_t]_{\#} p_0 \coloneqq p_0(\psi_t^{-1}(m)) \det \left\lvert \frac{\partial \psi_t^{-1}(m)}{\partial m} (m) \right\rvert$ for our choice of flat $\Mcal$ \cite{mathieu2020riemannian, gemici2016normalizing, falorsi2020neural}. \citet{chen2018neural} proposed to model $\psi_t$ implicitly by parameterizing $u_t$ to produce $p_t$ in a method known as a \emph{Continuous Normalizing Flow} (CNF).

\paragraph{Flow Matching}
Fitting a CNF using maximum likelihood, as in the style of \cite{chen2018neural}, can be expensive and unstable. A more effective alternative, fitting a vector field $v_t^{\theta} \in \Ucal$ with parameters $\theta$, may be accomplished by doing regression on vector fields $u_t$ that are known \emph{a priori} to generate $p_t$. The method is known as \emph{Flow Matching} \cite{lipman2022flow} and was extended to $\Mcal$ by \citet{chen2023riemannian}. \citet{lipman2022flow} note that $u_t$ is generally intractable and formulate an alternative objective based on tractable, conditional vector fields $u_t(m \mid m_1)$ that generate conditional probability paths $p_t(m \mid m_1)$, the push-forward of the conditional flow $\psi_t(m \mid m_1)$, starting at any $p$ and concentrating around $m=m_1 \in \Mcal$ at $t=1$. Marginalizing over target distribution $q$ recovers the unconditional probability path $p_t(m)=\int_\Mcal p_t(m \mid m_1) q(m_1) \, dm_1$ and the unconditional vector field $u_t(m) = \int_\Mcal u_t(m \mid m_1) \frac{p_t(m \mid m_1) q(m_1)}{p_t(m)} \, dm_1$. This construction results in an unconditional path $p_t$ where $p_0 =p$, the chosen base distribution, and $p_1 = q$, the target distribution. Their proposed objective, simplified for flat manifolds $\Mcal$, is:
\begin{align}
\label{eqn:conditional-flow-matching-objective}
    \Lfrak(\theta) = \Ebb_{t, q(m_1), p_t(m \mid m_1)} \lVert v_t^{\theta}(m) - u_t(m \mid m_1) \rVert^2,
\end{align}
where the $\lVert \cdot \rVert$ norm is induced by inner product $\lA \cdot, \cdot \rA$ on $\Tcal_m \Mcal$ and $t \sim \text{Uniform}(0, 1)$. At optimum, $v_t^{\theta}$ generates $p_t^{\theta} = p_t$ with endpoints $p_{0}^{\theta}=p$, $p_{1}^{\theta}=q$. At inference, we sample $p$ and propagate $t$ from 0 to 1 using our estimated $v_t^{\theta}$.

\subsection{Crystalline Solids}
\label{sec:crystalline_solids}

\paragraph{Stability and the convex hull}
One of the most important properties of a material is its stability, a heuristic that gives a strong indication of its synthesizability. A crystal is stable when it is energetically favorable compared with competing \emph{phases}, structures built from the same atomic constituents, but in different proportion or spacial arrangement. The energy can be computed using a first-principles quantum mechanical method called density functional theory, which estimates the energy based on the electronic structure \cite{kohn1965self}. The lowest energy materials form a convex hull over composition. Stable structures lie directly on the convex hull or below it, while meta-stable structures are restricted to $E^{hull} < 0.08$ eV/atom. Note that, this definition has inherent epistemic uncertainty since many materials are unknown and not represented on the convex hull. Our specific convex hull is in reference to the Materials Project database, as recorded by \citet{riebesell2024convexhull} in February 2023. 

\paragraph{Arity for materials}
A material with $N$ unique atom type constituents is known as an \emph{$N$-ary} material and its stability is determined by an $N$-dimensional convex hull. High $N$-ary materials occupy convex hulls that not represented in the realistic datasets under consideration; the curse of dimensionality and chemical complexity limits coverage of these hulls. We posit that an effective generative model of stable structures will produce a distribution over $N$-ary that is close to the data distribution. This is borne out in our experiments as seen in Figure~\ref{fig:nary_dist} and in Appendix~\ref{appendix:results_continued}. On another note, several generative models produced $1$-ary structures marked stable by the $E^{hull}<0$ criterion. This is not possible as the one-element phase diagrams are known and there are no energetically favorable structures to be found. This is a numerical issue; we did not count those structures as stable.

\subsection{Problem statements \& Datasets}
\label{sec:problem_statement_data}

\paragraph{Crystal Structure Prediction (CSP)} 
We aim to predict the stable structure for a given composition $\abold$, but this is not well-posed because some compositions have no stable arrangement. Furthermore, the underlying energy calculations have uncertainty and some structures are degenerate. We therefore formulate CSP as a conditional generative task on metastable structures, distributed like $q(\fbold, \lbold; \abold)$, rather than via regression on stable structures. Our dataset consists of metastable structures, indexed by composition and count:
\begin{equation}
	\label{eqn:crystal-structure-prediction}
	\fbold_{ij}, \lbold_{ij} \in \lC \fbold' \in \Fcal, \lbold' \in \Lcal \mid E(\abold_i, \fbold', \lbold') < E_{\text{m}} \rC.
\end{equation}
$E_{\text{m}} \coloneqq 0.08$ eV/atom is fixed by metastability, $E \colon \Ccal \to \Rbb$ is the single point energy prediction of density functional theory, $\abold_i$ is a composition with index $i$, and $\fbold'_{ij}, \lbold'_{ij}$ are the corresponding metastable structures with index $j$. The maximum values of $i$ and $j$ are the number of metastable compositions and structures for that composition, respectively.

In CSP, we fit a generative model $p(\fbold, \lbold \mid \abold)$ to the samples $\fbold_{ij}, \lbold_{ij} \sim q(\fbold, \lbold; \abold_i)$. Given the right $\abold$, a good generative model should generate corresponding metastable structures. 

\paragraph{De novo generation (DNG)} 
A major goal of materials science is to discover stable and novel crystals. In this effort, we aim to sample directly from a distribution of metastable materials $q(\cbold)$, generating both the structure $\fbold, \lbold$ along with the composition $\abold$. Our distribution must include $\abold$ because it should avoid compositions that have no (meta)stable structure and because many new and interesting materials may have novel compositions. Define our dataset:
\begin{equation}
	\label{eqn:de-novo-generation}
	\abold_k, \fbold_k, \lbold_k \coloneqq \cbold_k \in \lC \cbold' \in \Ccal \mid E(\cbold') < E_{\text{m}} \rC,
\end{equation}
consisting of $\max k$ metastable crystals.

In DNG, we fit a generative model $p(\cbold)$ to samples $\cbold_k \sim q(\cbold)$. A good generative model should generate both novel and known metastable materials. Determining stability and novelty requires further computation and a convex hull.

\paragraph{Practical Considerations \& Data}
We consider two realistic datasets: \emph{MP-20}, containing all materials with a maximum of 20 atoms per unit cell and within 0.08 eV/atom of the convex hull in the Materials Project database from around July 2021 \cite{jain2013materials}, and \emph{MPTS-52}, a challenging dataset containing structures with up to 52 atoms per unit cell and separated into ``time slices'' where the training, validation, and test sets are organized chronologically by earliest published year in literature \cite{baird2024mpts}.

We include two additional datasets \emph{Perov-5} \cite{castelli2012new} and \emph{Carbon-24} \cite{pickard2020cabon} as unit tests. These do not feature crystals near their energy minima. Perov-5 consists of crystals with varying atomic types, but all structures have the same fractional coordinates. Carbon-24 structures take on many arrangements in fractional coordinates, but only consist of one atom type. Stability analysis is not applicable to Perov-5 and Carbon-24, but proxy metrics introduced by \citet{xie2021crystal} are applicable to all datasets.

\paragraph{Standard (proxy) metrics for CSP and DNG}
Although computing the stability for generated materials is ideal, it is extremely expensive and technically challenging. In light of these difficulties, a number of \emph{proxy metrics} have been developed by \citet{xie2021crystal}. The primary advantage of these metrics is their low cost. We benchmark FlowMM and alternatives using specialized metrics for CSP and DNG.

In CSP we compute the \emph{match rate} and the \emph{Root-Mean-Square Error (RMSE)}. They measure the percentage of reconstructions from $q(\fbold, \lbold; \abold)$ that are satisfactorily close to the ground truth structure and the RMSE between coordinates, respectively. In DNG, we compute a \emph{Structural \& Compositional Validity} percentage using heuristics about interatomic distances and charge, respectively. We also compute \emph{Coverage Recall \& Precision} on chemical fingerprints and the \emph{Wasserstein distance} between ground truth and generated material properties, namely atomic density $\rho$ and number of unique elements per unit cell $N_{el}$. Note $N_{el} = N\text{-ary}$. See Appendix~\ref{appendix:preliminaries} for more details.

\paragraph{Stability metrics for DNG}
The ultimate goal of materials discovery is to propose stable, unique, and novel materials efficiently w.r.t. compute. For flow matching and diffusion models, the most expensive inference-time cost is integration steps. We therefore define several metrics to address these factors on a budget of 10,000 generations.

We compute the percent of generated materials that are stable (\emph{Stability Rate}), but that does not address novelty. Following \citet{zeni2023mattergen}, we additionally compute the percentage of \emph{stable, unique, and novel (S.U.N.)} materials (\emph{S.U.N. Rate}). To address cost, we compute the average number of integration steps needed to generate a stable material (\emph{Cost}) and a S.U.N. material (\emph{S.U.N.} Cost). We explain identification of S.U.N. materials in Appendix~\ref{appendix:preliminaries}.

\section{Riemannian Flow Matching for Materials}
\label{sec:rfm4mat}

Our goal is to define a parametric generative model on the Riemannian manifold $\Ccal$ that carries the geometry and invarinces inherent to crystals. We plan to accommodate both CSP and DNG with simple changes to our model and base distribution.

Concretely, we have a set of samples $\abold_1, \fbold_1, \lbold_1 \sim q(\abold, \fbold, \lbold)$ where $q \in \Pcal$ is an unknown probability distribution over $\Ccal$, and we want to implicitly estimate the probability path $p_t: \lB 0, 1 \rB \times \Pcal \to \Pcal$ that transforms our chosen base distribution $p_0 = p \in \Pcal$ to $p_1 = q$. We achieve this by learning a parametric time-dependent vector field $v_t^{\theta}$ with parameters $\theta$ that optimizes \eqref{eqn:conditional-flow-matching-objective}, adapted to $\Ccal$, on samples $\abold_1, \fbold_1, \lbold_1$. Additionally, we enforce that the unconditional probability path be invariant to symmetries $g \in (\sigma, Q, \tau)$ at all times $t$:
\begin{align}
    p_{t}(g \cdot \cbold) &\coloneqq \int_\Ccal p_t(g \cdot \cbold \mid \cbold_1) \, q(\cbold_1) \, d\cbold_1 = p_{t}(\cbold). \label{eqn:unconditional_probability_path_is_invariant}
\end{align}
We explain the necessary building blocks of our method (a) the geometry of $\Ccal$, (b) the base distribution $p$ and its invariances, (c) the conditional vector fields and our objective derived from \eqref{eqn:conditional-flow-matching-objective}, and finally (d) how our construction affirms that the marginal probability path $p_t(\cbold)$ generated by $u_t(\cbold)$ has the invariance properties of $q$ for all $t \in [0,1]$.

\paragraph{Geometry of $\Ccal$}
Recall $\Ccal \coloneqq \Acal \times \Fcal \times \Lcal$ forms a product manifold, implying that the inner product $\lA \cbold, \cbold' \rA$ decomposes with addition, see Appendix~\ref{appendix:preliminaries}. We now consider the geometry of each component of $\Ccal$ individually.

$\Fcal$ is a permutation invariant collection of $n \times 3$-dimensional flat tori. That means it carries the Euclidean inner product locally, but each side is identified with its opposite. This is relevant for the geodesic and explains why paths ``wrap around'' the domain's edges. The identification is implemented by the action of the translation operator defined in Section~\ref{sec:preliminaries}.

The space of lattice parameters subject to the Niggli reduction $\Lcal \coloneqq \Rbb^{+3} \times \lB 60, 120 \rB^{3}$ is Euclidean, but has boundaries. We can safely ignore these boundaries for the lengths in $\Rbb^{+3}$ since (i) the data does not lie on the boundary ($a,b,c > 0$) and (ii) we select a positive base distribution. Meanwhile, the angles $\alpha, \beta, \gamma$ do often lie directly on the boundary, posing a problem as the target $u_t$ is not a \emph{smooth} vector field. We address the issue with a diffeomorphism $\varphi \colon [60, 120] \to \Rbb$ to \emph{unconstrained space}, applied element wise to $\alpha, \beta, \gamma$: 
\begin{align*}
    \text{logit}(\xi) &\coloneqq \log \frac{\xi}{1 - \xi}, & \varphi(\eta) &\coloneqq \text{logit} \lp\frac{\eta - 60}{120}\rp, \\
    \Sfrak(\xi') &= \frac{\exp(\xi')}{1 + \exp(\xi')}, & \varphi^{-1}(\eta') &= 120 \, \Sfrak \lp \eta' \rp + 60,
\end{align*}
where $\Sfrak$ is the sigmoid function. Practically, geodesics and conditional vector field $u_t$ are both represented in unconstrained space. Base samples and (estimated) target samples are transformed into unconstrained space for learning and integration then evaluated in $\lB 60, 120 \rB$, transformed with $\varphi^{-1}$.

The details of $\Acal$ depend on whether our task is CSP, where we estimate $q(\fbold, \lbold; \abold)$ or DNG, where we estimate $q(\abold, \fbold, \lbold)$. In CSP, $\abold$ is given and the geometry is simple: $\Acal$ is a $h$-dimensional one-hot vector. Components of its unconditional vector field $u_t^{\Acal}(\abold)=v_t^{\Acal, \theta}(\abold)=0$ everywhere. Further discussion about $\Acal$ for CSP is unnecessary in this section and thus omitted! When doing DNG, we take $\Acal$ to be a $\lceil \log_2 h \rceil $-dimensional Euclidean space with a flat metric. Here, $\Acal$ has a simple geometry but the interesting part is that it represents atomic types more efficiently than a one-hot vector, in terms of dimension, after discretization with $\sign$.

\paragraph{Base distribution on $\Ccal$}
Our base distribution on $\Ccal$ is a product of distributions on $\Acal$, $\Fcal$, and $\Lcal$, see Appendix~\ref{appendix:preliminaries}. The base distribution $p(\abold)$ takes the same base distribution as \citet{chen2022analog}, namely $p(\abold) = \Ncal(\abold; 0, 1)$ where $\Ncal$ denotes the normal distribution.
Next, we choose the distribution over $\Fcal$ to be $p(\fbold) \coloneqq \text{Uniform}(0, 1)$, which covers the torus with equal density everywhere.
Finally, we decided to leverage the flexibility afforded by Flow Matching in choosing an \emph{informative} base distribution on $\lbold$. Recall, $\lbold$ can be split into three length and three angle parameters. Since length parameters are all positive we set $p(a, b, c) \coloneqq \prod_{\eta \in \{ \alpha, \beta, \gamma \}} \text{LogNormal}(\eta; \text{loc}_{\eta}, \text{scale}_{\eta})$ where $\text{loc}_{\eta}, \text{scale}_{\eta}$ are fit to training data using maximum likelihood. In the constrained space $\lB 60, 120 \rB$, angle parameters $\alpha, \beta, \gamma$ get base distribution $p(\alpha, \beta, \gamma) \coloneqq \text{Uniform}(60, 120)$. Samples can be drawn in unconstrained space by applying $\varphi$ and the density can be computed using the change of variables formula.

The base distributions $p(\abold)$ and $p(\fbold)$ are factorized and have no dependency on index. Therefore, they are permutation invariant.
Next, $p(\fbold)$ is translation invariant since,
\begin{align}
    p(\tau \cdot \fbold) &= \prod_{i = 1, \ldots, n} U(f^{i} + \tau - \lfloor f^{i} + \tau \rfloor; 0, 1) \nonumber \\
    &= \prod_{i = 1, \ldots, n} U(f^{i}; 0, 1) = p(\fbold),
\end{align}
for all translations $\tau \in [-\tfrac{1}{2}, \tfrac{1}{2}]^{3}$.
Our base distribution is $p(\abold, \fbold, \lbold) \coloneqq p(\abold)p(\fbold)p(\lbold)$, so it carries these invariances. It remains to be shown that $\psi_t$ is equivariant to these groups.

\paragraph{Conditional vector fields on $\Ccal$}
Recall from \citet{chen2023riemannian}, the conditional vector field on flat $\Mcal$ is
\begin{align}
\label{eqn:conditional_vector_field}
    u_t(m \mid m_1) = \frac{d \log \kappa(t)}{dt} d(m, m_1) \frac{\nabla_{m} d(m, m_1)}{\rVert \nabla_{m} d(m, m_1) \lVert^2}
\end{align}
where $d: \Mcal \times \Mcal \to \Rbb$ is the geodesic distance \eqref{eqn:geodesic} and $\kappa(t) = 1 - t$ is a linear time scheduler. Both $\Acal$ for DNG and (transformed) $\Lcal$ are Euclidean manifolds with standard norm, recovering the Flow Matching conditional vector field on their respective tangent bundles, which we denote $u_t^{\Mcal'}(m \mid m_1) = \frac{m_1 - m}{1 - t}$ for $\Mcal' \in \{ \Acal, \Lcal \}$. 

We construct the conditional vector field for a point cloud living on a $n \times 3$-dimensional flat torus invariant to global translations. First, we construct the naive geodesic path, which may cross the periodic boundary:
\begin{align}
    \exp_{f^{i}}(\dot{f^{i}}) &\coloneqq f^{i} + \dot{f^{i}} - \lfloor f^{i} + \dot{f^{i}} \rfloor, \\
    \log_{f^{i}_0}(f^{i}_1) &\coloneqq \frac{1}{2\pi} \atantwo\lB\sin(\omega^i), \cos(\omega^i) \rB, \label{eqn:atom_wise_torus_logmap} \\
    \omega^i &\coloneqq 2 \pi(f^{i}_1 - f^{i}_0), \label{eqn:torus_angular_frequency_difference}
\end{align}
where $\dot{f^{i}} \in \Tcal_{f^{i}}\Fcal^{i}$ for $i = 1, \ldots, n$. These amount to an atom wise application of $\log_{\fbold_0}$ on $\fbold_1$ and $\exp_{\fbold}$ on $\dot{\fbold} \in \Tcal_{\fbold}\Mcal$ respectively. Specifically, $d(\fbold, \fbold_1) \coloneqq \lVert \log_{\fbold_1}(\fbold) \rVert^{2}$ and 
$\nabla_{\fbold} d(\fbold, \fbold_1) = -2\log_{\fbold_1}(\fbold)$.
That would imply a target conditional vector field of $\frac{-\log_{\fbold_1}(\fbold)}{1 - t}$: a function which is equivariant--\emph{not invariant}--to translation $\tau$! We address this by removing the average torus translation from $\fbold_1$ to $\fbold$:
\begin{equation}
\label{eqn:translation_invariant_conditional_vf}
    u_t^{\Fcal}(\fbold \mid \fbold_1) \coloneqq \log_{\fbold_1}(\fbold) - \frac{1}{n} \sum_{i=1}^{n} \log_{f^{i}_1}(f^{i}).
\end{equation}
Our approach is similar to subtracting the mean of a point cloud in Euclidean space; however, it occurs in the tangent space instead of on the manifold. Given the factorization of the inner product on $\Ccal$ (Appendix~\ref{appendix:preliminaries}), our objective is:
\begin{align}
    &\Ebb_{t, q(\abold_1, \fbold_1, \lbold_1), p(\abold_0), p(\fbold_0), p(\lbold_0)} \Bigl[
    \frac{\lambda_{\abold}}{h n} \left\lVert v_t^{\Acal, \theta}(\cbold_t) + \abold_0 - \abold_1 \right\rVert^2 \nonumber \\
    & + \frac{\lambda_{\fbold}}{3n} \left\lVert v_t^{\Fcal, \theta}(\cbold_t) + \log_{\fbold_1}(\fbold_0) - \frac{1}{n} \sum_{i=1}^{n} \log_{f^{i}_1}(f^{i}_0) \right\rVert^2 \label{eqn:objective} \\ 
    & + \frac{\lambda_{\lbold}}{6} \left\lVert v_t^{\Lcal, \theta}(\cbold_t) + \lbold_0 - \lbold_1 \right\rVert^2 \Bigr], \nonumber
\end{align}
where we've normalized by dimension and $\lambda_{\abold}, \lambda_{\fbold}, \lambda_{\lbold} \in \Rbb^+$ are hyperparameters and $t\sim\text{Uniform}(0,1)$. In practice, since we have a closed form geodesic for all of our manifolds, our supervision signals are computed by evaluating conditional flow $\psi_t(\cbold \mid \cbold_1)$ on a minibatch determined by \eqref{eqn:geodesic} at time $t$ and taking the gradient with automatic differentiation, and in the component on $\Fcal$ we subtract the mean.

\paragraph{Symmetries of the marginal path}
We show the symmetries of the conditional probability paths and construct the marginal path. Conditional probability path $p_t(\cbold \mid \cbold_1)$ mapping $p_0=p(\cbold)$ to $p_1(\cbold \mid \cbold_1)$ is generated by tuple $u_t(\cbold \mid \cbold_1) \coloneqq \lp u_t^{\Acal}(\abold \mid \abold_1), u_t^{\Fcal}(\fbold \mid \fbold_1), u_t^{\Lcal}(\lbold \mid \lbold_1) \rp$ formed by direct sum. Conditional vector fields $u_t^{\Acal}$ and $u_t^{\Fcal}$ are permutation equivariant through relabeling of particles and therefore $p_t(\abold \mid \abold_1)$ and $p_t(\fbold \mid \fbold_1)$ are invariant to permutation by \citet{kohler2020equivariant}. The representation of $\lbold$ makes $p_t(\lbold \mid \lbold_1)$ invariant to rotation. Finally, by translating away the mean tangent fractional coordinate we relaxed the traditional requirement in Flow Matching that conditional path $p_1(\fbold \mid \fbold_1) = \delta(\fbold - \fbold_1)$ and instead allow $p_1$ to concentrate on an equivalence class of $\fbold_1$ where all members in the same class are related by a translation. Therefore, $p_1(\fbold \mid \fbold_1)$ remains translation equivariant but the marginal probability path ends up translation invariant  (Theorem~\ref{thm:invariant_flow}). 
Our unconditional vector field $u_t$, generating unconditional probability path $p_t$, connecting $p_0=p$ to $p_1=q$ is:
\begin{align}
    u_t(\cbold) &\coloneqq \int_\Ccal u_t(\cbold \mid \cbold_1) \frac{p_t(\cbold \mid \cbold_1)q(\cbold_1)}{p_t(\cbold)} \, d\cbold_1 \label{eqn:unconditional_vector_field} \\
    p_{t}(\cbold) &\coloneqq \int_\Ccal p_t(\cbold \mid \cbold_1) q(\cbold_1) \, d\cbold_1. \label{eqn:unconditional_probability_path}
\end{align}
Our construction enforces that $p_{t}(\cbold)$ is invariant to permutation, translation, and rotation, with proof in Appendix~\ref{appendix:marginal_probability_path_invariance}.

\paragraph{Estimated marginal path}
We specify our model, the unconditional probability path $p_t^{\theta}(\cbold)$, generated by $v_t^{\theta}(\cbold)$, 
\begin{align}
    p_{t}^{\theta}(\cbold) \coloneqq \int_\Ccal p_t^{\theta}(\cbold \mid \cbold_1) q(\cbold_1) d\cbold_1,
\end{align}
trained by optimizing \eqref{eqn:objective}, and when $p_0 = p$ then $p_1 \approx q$. We let $v_t^{\theta}(\cbold)$ be a graph neural network in the style of \cite{satorras2021n, jiao2023crystal} that enforces equivariance to permutation for $v_t^{\Acal, \Fcal, \theta}(\abold, \fbold)$ via message passing and invariance to translation in $v_t^{\Fcal, \theta}(\fbold)$ by featurizing graph edges as displacements between atoms. Invariance to rotation of $v_t^{\Lcal, \theta}(\lbold)$ is enforced by the representation of $\lbold$. After enforcing these symmetries in our network, we know that $p_{t}^{\theta}(\cbold)$ has the invariances desired by design. For more details about the graph neural network, see Appendix~\ref{appendix:neural_network}.

\paragraph{Inference Anti-Annealing}
A numerical trick which increased the performance of our neural network on the proxy metrics was to adjust the predicted velocity $v_t^{\theta}(\cbold)$ at inference time. We write the ordinary differential equation and initial condition $\cbold \sim p(\cbold)$ that defines the flow,
\begin{align}
    \label{eqn:anti-annealing}
    \frac{d}{dt} \psi_t^{\theta} &= s(t) v_{t}^{\theta}(\psi_t^{\theta}(\cbold)), & \psi_0^{\theta}(\cbold) &= \cbold,
\end{align}
but include a time-dependent velocity scaling term $s(t) \coloneqq 1 + s't$ where $s'$ is a hyperparameter. We typically found best performance when $0 \leq s' \leq 10$. Notably, we also found that selectively applying the velocity increase to particular variables had a significant effect. In CSP, it was helpful for fractional coordinates but hurtful for lattice parameters $\lbold$. We increased the fractional coordinate velocity for our reported results. For DNG, the trend was not as simple. Further investigation of this effect through ablation study can be found in Appendix~\ref{appendix:results_continued}.
This effect has been observed in multiple other studies \cite{yim2023fast, bose2023se}.

\section{Experiments}
\label{sec:experiments}

We evaluate FlowMM on the two tasks we set out at the beginning of the paper: Crystal Structure Prediction and De Novo Generation. We apply proxy metrics in all experiments, with a focus on inference-stage efficiency. We additionally investigate the stability of DNG samples by performing extensive density functional theory calculations.

\subsection{Crystal Structure Prediction}
\label{sec:crystal_structure_prediction}
We perform CSP on all datasets (Perov-5, Carbon-24, MP-20, and MPTS-52) with CDVAE, DiffCSP and FlowMM, evaluating them with proxy metrics computed using \texttt{StructureMatcher} \cite{ong2013python}. We present the Match Rate and the Root-Mean-Square Error (RMSE) in Table~\ref{table:conditional_metrics}.
DiffCSP and FlowMM use exactly the same underlying neural network in an apples-to-apples comparison. FlowMM outperforms competing models on the more challenging \& realistic datasets (MP-20 and MPTS-52) on both metrics by a considerable margin. Figure~\ref{fig:match_rate_per_step} investigates sampling efficiency by comparing the match rate of DiffCSP and FlowMM as a function of number of integration steps. FlowMM achieves a higher match rate with far fewer integration steps, which corresponds to more efficient inference. FlowMM achieves maximum match rate in about 50 steps, at least an order of magnitude decrease in inference time cost compared to the 1000 steps used by DiffCSP. We ablate both inference anti-annealing and the proposed base distribution $p(\lbold)$ and confirm that FlowMM is competitive or outperforms other models in terms of Match Rate without those enhancements. We additionally report inference-time uncertainty. Those results are located in Appendix~\ref{appendix:results_continued}.
\begin{figure}[ht]
    \vskip 0.0in
    \begin{center}
    \centerline{\includegraphics[width=\columnwidth]{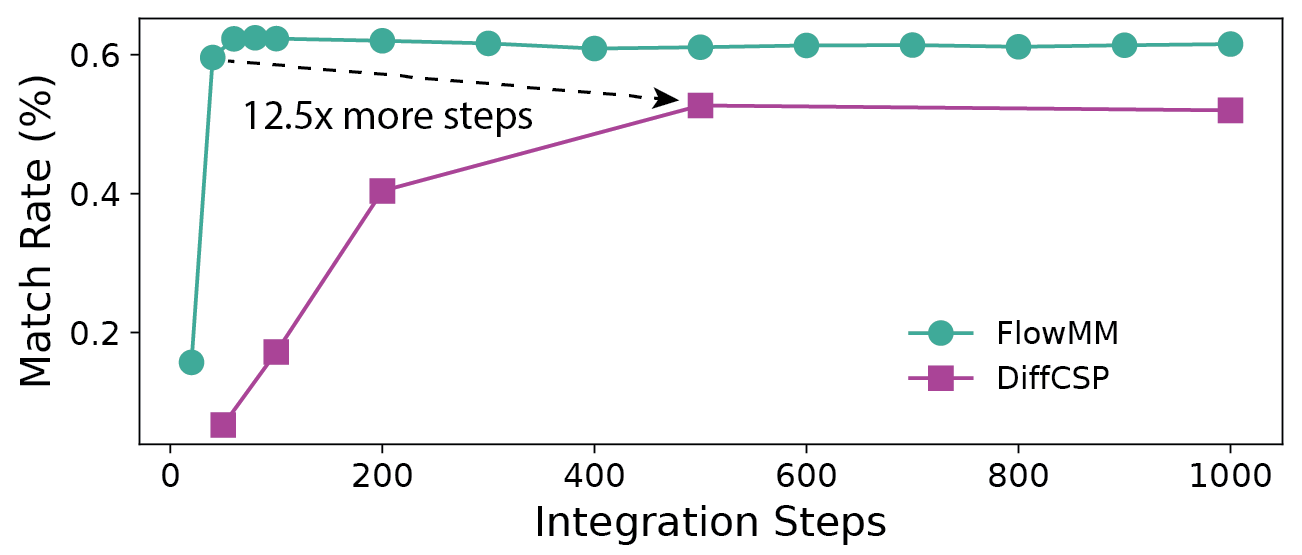}}
    \caption{
    Match rate as a function of number of integration steps on MP-20. FlowMM achieves a higher maximum match rate than DiffCSP overall, and does so $\sim$ 450 steps before DiffCSP. Results with Inference Anti-Annealing ablated are in Appendix~\ref{appendix:results_continued}.
    }
    \label{fig:match_rate_per_step}
    \end{center}
    \vskip -0.4in
\end{figure}
\begin{table*}
\centering
\caption{Results from crystal structure prediction on unit tests and realistic data sets.
}
\label{table:conditional_metrics}
\resizebox{0.95\textwidth}{!}{
\begin{tabular}{lcccccccc}
\toprule
 & \multicolumn{2}{c}{Perov-5} & \multicolumn{2}{c}{Carbon-24} & \multicolumn{2}{c}{MP-20} & \multicolumn{2}{c}{MPTS-52} \\
 & Match Rate (\%) $\uparrow$ & RMSE $\downarrow$ & Match Rate (\%) $\uparrow$ & RMSE $\downarrow$ & Match Rate (\%) $\uparrow$ & RMSE $\downarrow$ & Match Rate (\%) $\uparrow$ & RMSE $\downarrow$ \\
\midrule
CDVAE &  45.31 &  0.1138 & 17.09 & 0.2969 & 33.90 & 0.1045 & 5.34 & 0.2106 \\
DiffCSP & 52.02 & \textbf{0.0760} & 17.54 & \textbf{0.2759} & 51.49 & 0.0631 & 12.19 & 0.1786 \\
FlowMM & \textbf{53.15} & 0.0992 & \textbf{23.47} & 0.4122 & \textbf{61.39} & \textbf{0.0566} & \textbf{17.54} & \textbf{0.1726} \\
\bottomrule
\end{tabular}}
\end{table*}

\begin{table*}
\centering
\caption{Results from De Novo generation on the MP-20 dataset.}
\label{table:unconditional_metrics}
\resizebox{0.95\textwidth}{!}{
\begin{tabular}{cccccccccccc}
\toprule
 Method & Integration & \multicolumn{2}{c}{Validity (\%) $\uparrow$} & \multicolumn{2}{c}{Coverage (\%) $\uparrow$} & \multicolumn{2}{c}{Property $\downarrow$} & Stability Rate$^\dag$ (\%) $\uparrow$ & Cost $\downarrow$ & S.U.N. Rate $\uparrow$ & S.U.N. Cost $\downarrow$ \\
  &  Steps & Structural & Composition & Recall & Precision & wdist ($\rho$) & wdist ($N_{el}$) & MP-2023 & Steps/Stable$^\dag$ & MP-2023 & Steps/S.U.N. \\
\midrule
CDVAE & 5000 & \textbf{100.00} & \textbf{86.70} & 99.15 & 99.49 & 0.688 & 0.278 & 1.57 & 31.85 & 1.43 & 34.97 \\
DiffCSP & 1000 & \textbf{100.00} & 83.25 & \textbf{99.71} & \textbf{99.76} & 0.350 & 0.125 & \textbf{5.06} & 1.98 & \textbf{3.34} & 2.99 \\
\midrule
\multirow{4}{*}{FlowMM} & 250 & 96.58 & 83.47 & 99.48 & 99.65 & \textbf{0.261} & \textbf{0.107} & 4.32 & \textbf{0.58} & 2.38 & \textbf{1.05} \\
  & 500 & 96.86 & 83.24 & 99.38 & 99.63 & \textbf{0.075} & \textbf{0.079} & 4.19 & \textbf{1.19} & 2.45 & \textbf{2.04} \\
  & 750 & 96.78 & 83.08 & 99.64 & 99.63 & \textbf{0.281} & \textbf{0.097} & 4.14 & \textbf{1.81} & 2.22 & 3.38 \\
  & 1000 & 96.85 & 83.19 & 99.49 & 99.58 & \textbf{0.239} & \textbf{0.083} & 4.65 & 2.15 & 2.34 & 4.27 \\
\bottomrule
\bigskip
Stable$^\dag$ implies & \hspace{-1em} $E^{hull} < 0.0$ & \hspace{-1em} \& $N$-ary $\geq$ 2.
\end{tabular}%
}
\vskip -0.2in
\end{table*}
\begin{figure}[ht]
    \vskip 0.0in
    \begin{center}
    \centerline{\includegraphics[width=\columnwidth]{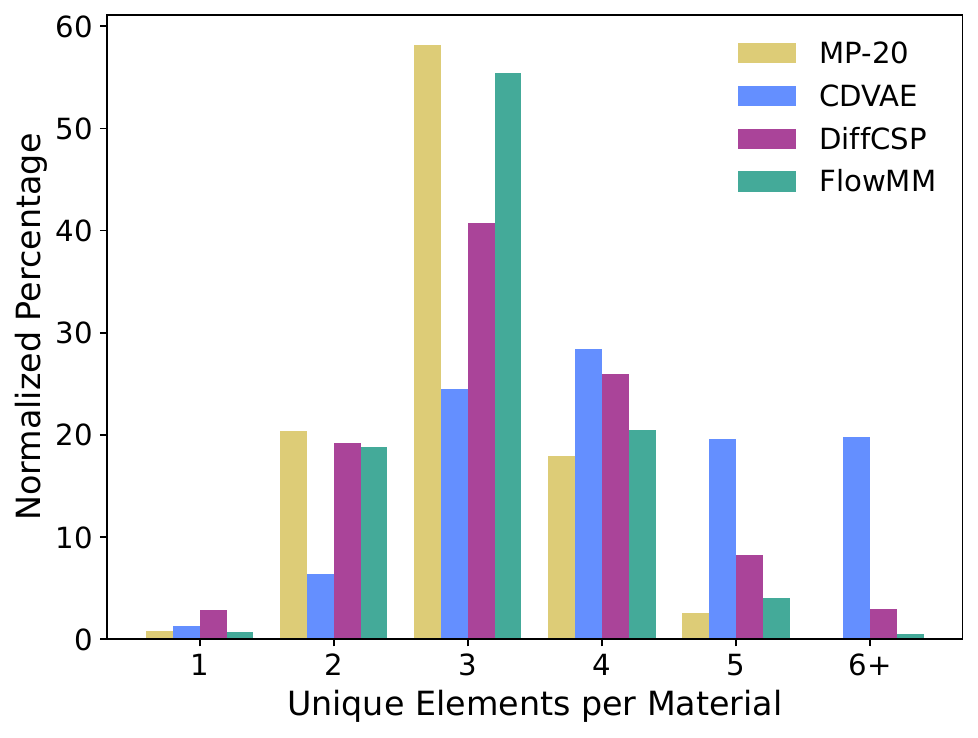}}
    \caption{The distribution of number of unique elements per material, or $N$-ary, for the MP-20 distribution and the generative models. FlowMM matches the MP-20 distribution closest, while CDVAE and DiffCSP generate too many materials with $N$-ary $\geq$ 5.}
    \label{fig:nary_dist}
    \end{center}
    \vskip -0.4in
\end{figure}

\subsection{De novo generation}
\label{sec:unconditional_generation}

To evaluate de novo generation we trained models on the MP-20 dataset and we generated 10,000 structures from CDVAE and DiffCSP. For FlowMM, we generated 40,000 structures, in batches of 10,000, using a variable number of integration steps: $\{250, 500, 750, 1000\}$.
Table~\ref{table:unconditional_metrics} shows the proxy metrics along with the stability metrics computed using the generated structures.
FlowMM is competitive with the diffusion models on most metrics, but significantly outperforms them on several Wasserstein distance metrics between distributions of properties of generated structures and the test set, specifically: the atomic density $\rho$ and the number of unique elements per crystal $N_{el}$ (same as $N$-ary).

Table~\ref{table:unconditional_metrics} also shows two different stability metrics based on energy above hull ($E^{hull}$) calculations. To compute $E^{hull}$ for the experiments, we ran structure relaxations with CHGNet \cite{deng2023chgnet} and density functional theory and used those to determine the distance to the convex hull and thereby stability. Further details are in Appendix~\ref{appendix:preliminaries}.
Conventional methods \cite{glass2006uspex, pickard2011ab} involve random search and hundreds of expensive density functional theory evaluations. We aim to reduce the computational expense of De Novo Generation. Therefore, S.U.N. Cost is our most important metric as it indicate the expense of finding a new material in terms of integration steps at inference time. From Table~\ref{table:unconditional_metrics}, it is clear that FlowMM is competitive to DiffCSP on the S.U.N. Rate and Stablity Rate metrics, but significantly better on Cost and S.U.N. Cost due to the reduction in integration steps. This efficiency is typical of flow matching compared to diffusion \cite{shaul2023kinetic,yim2023fast,bose2023se}. We note the caveat that integration steps are not the only source of computational cost. Training, prerelaxation, and relaxation are all costs worth considering; however, they are slightly more difficult to benchmark. Furthermore, we found them to be approximately equal across models so we focus on the cost of inference, which varies considerably.

We also compare the distribution of the computed $E^{hull}$ values for the various methods in Figure \ref{fig:energy_generation}. 
Structures generated by FlowMM are on average much more stable than CD-VAE, and are comparable to those generated by DiffCSP. %

We compare the distribution of materials according to the arity of the structure. Figure~\ref{fig:nary_dist} compares the $N$-ary distribution of each of the models to the MP-20 dataset. FlowMM matches the data distribution significantly better than the diffusion models, this is confirmed numerically with the Wasserstein distance metric $N_{el}$ in table \ref{table:unconditional_metrics}. We present a similar distribution for stable structures in Figure~\ref{fig:nary_stable_dist}.

\begin{figure}[ht]
    \vskip 0.0in
    \begin{center}
    \centerline{\includegraphics[width=\columnwidth]{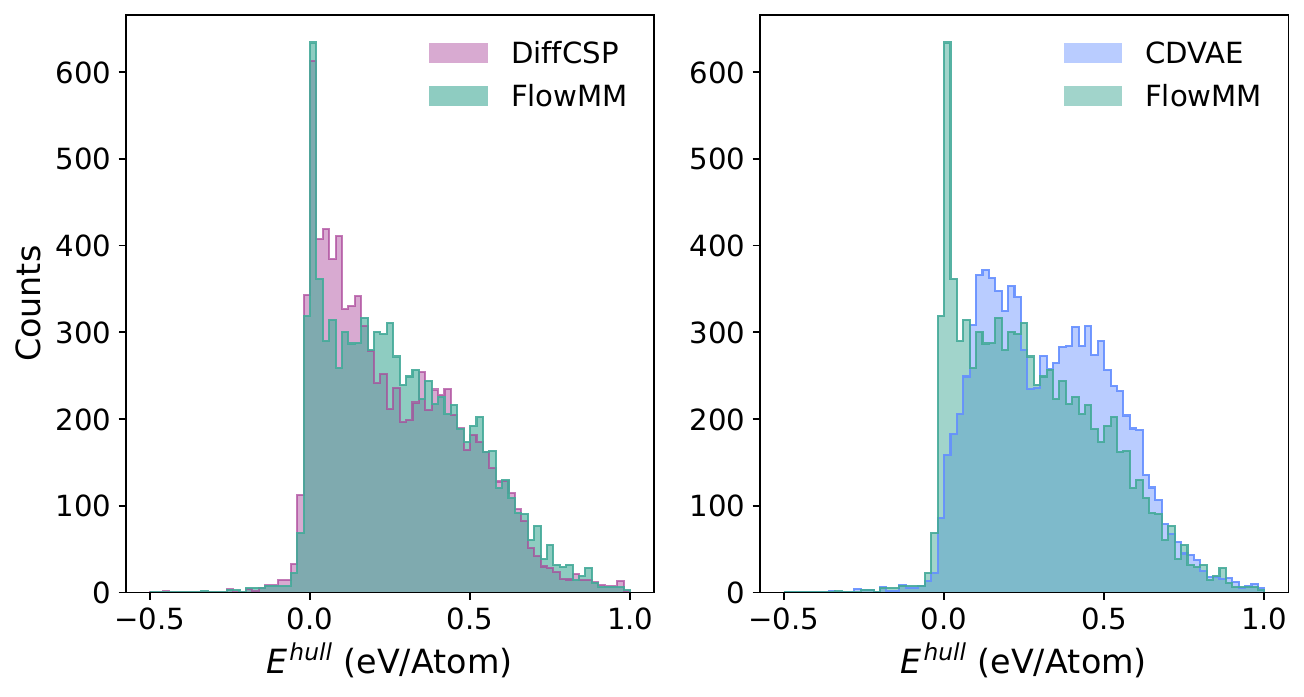}}
    \caption{Histogram comparing the distribution of $E^{hull}$ computed after relaxation with DFT for generative models DiffCSP and CDVAE with our proposed FlowMM on the DNG task. After relaxation on for all models, FlowMM generates lower energy structures compared to CDVAE and is competitive with DiffCSP.}
    \label{fig:energy_generation}
    \end{center}
    \vskip -0.4in
\end{figure}

\section{Conclusion}
\label{sec:conclusion}

We introduced a novel method for training continuous normalizing flows using a generalization of Riemannian Flow Matching for generating periodic crystal structures. We empirically tested our model using both CSP and DNG tasks and found strong performance on all proxy-metrics. In CSP, we used exactly the same network as DiffCSP, thereby performing an apples-to-apples comparison. For MP-20, FlowMM was able to outperform DiffCSP, in terms of Match Rate, with as few as 50 integration steps. This represents at least an order of magnitude improvement.

We rigorously evaluated the DNG structures for stability using the energy above hull to determine the Stability Rate, Cost, S.U.N. Rate, and S.U.N. Costs for each model. We found that FlowMM significantly outperforms both CDVAE and DiffCSP on Cost and S.U.N. Cost, and is competitive with DiffCSP on Stability Rate and S.U.N. Rate. 
This is enabled by FlowMM's 3x more efficient generation, in terms of integration steps, at inference time. 
The inference time efficiency can be explained by the kinetically optimal paths learned using the Flow Matching objective \cite{shaul2023kinetic}.
Resource limitations meant we did not investigate whether FlowMM could generate a similar number of stable structures using only a handful of integration steps. Based on the extremely efficient CSP results in Figure~\ref{fig:match_rate_per_step}, this would be an interesting direction for future work.

\section*{Impact Statement}
Our paper presents a generative model for predicting the composition and structure of stable materials. Our work may aid in the discovery of novel materials that could catalyze chemical reactions, enable higher energy density battery technology, and advance other areas of materials science and chemistry. The downstream effects are difficult to judge, but the challenges associated with taking a computational prediction to synthesized structure imply the societal impacts are likely going to be limited to new research directions.

\section*{Acknowledgements}
The authors thank C. Lawrence Zitnick, Kyle Michel, Vahe Gharakhanyan, Abhishek Das, Luis Barroso-Luque, Janice Lan, Muhammed Shuaibi, Brook Wander, and Zachary Ulissi for helpful discussions. Meta provided the compute.

\bibliography{references}

\begin{thebibliography}{68}
\providecommand{\natexlab}[1]{#1}
\providecommand{\url}[1]{\texttt{#1}}
\expandafter\ifx\csname urlstyle\endcsname\relax
  \providecommand{\doi}[1]{doi: #1}\else
  \providecommand{\doi}{doi: \begingroup \urlstyle{rm}\Url}\fi

\bibitem[Adams \& Orbanz(2023)Adams and Orbanz]{adams2023representing}
Adams, R.~P. and Orbanz, P.
\newblock Representing and learning functions invariant under crystallographic
  groups.
\newblock \emph{arXiv preprint arXiv:2306.05261}, 2023.

\bibitem[AI4Science et~al.(2023)AI4Science, Hernandez-Garcia, Duval, Volokhova,
  Bengio, Sharma, Carrier, Koziarski, and Schmidt]{ai4science2023crystal}
AI4Science, M., Hernandez-Garcia, A., Duval, A., Volokhova, A., Bengio, Y.,
  Sharma, D., Carrier, P.~L., Koziarski, M., and Schmidt, V.
\newblock Crystal-gfn: sampling crystals with desirable properties and
  constraints.
\newblock \emph{arXiv preprint arXiv:2310.04925}, 2023.

\bibitem[Appl(1982)]{appl1982haber}
Appl, M.
\newblock The haber--bosch process and the development of chemical engineering.
  a century of chemical engineering, 1982.

\bibitem[Austin et~al.(2021)Austin, Johnson, Ho, Tarlow, and Van
  Den~Berg]{austin2021structured}
Austin, J., Johnson, D.~D., Ho, J., Tarlow, D., and Van Den~Berg, R.
\newblock Structured denoising diffusion models in discrete state-spaces.
\newblock \emph{Advances in Neural Information Processing Systems},
  34:\penalty0 17981--17993, 2021.

\bibitem[Baird et~al.(2024)Baird, Sayeed, and Riebesell]{baird2024mpts}
Baird, S.~G., Sayeed, H.~M., and Riebesell, J.
\newblock sparks-baird/matbench-genmetrics.
\newblock \url{https://github.com/sparks-baird/matbench-genmetrics}, 2024.
\newblock [Accessed 03-05-2024].

\bibitem[Bose et~al.(2023)Bose, Akhound-Sadegh, Fatras, Huguet, Rector-Brooks,
  Liu, Nica, Korablyov, Bronstein, and Tong]{bose2023se}
Bose, A.~J., Akhound-Sadegh, T., Fatras, K., Huguet, G., Rector-Brooks, J.,
  Liu, C.-H., Nica, A.~C., Korablyov, M., Bronstein, M., and Tong, A.
\newblock Se (3)-stochastic flow matching for protein backbone generation.
\newblock \emph{arXiv preprint arXiv:2310.02391}, 2023.

\bibitem[Cao et~al.(2024)Cao, Luo, Lv, and Wang]{cao2024space}
Cao, Z., Luo, X., Lv, J., and Wang, L.
\newblock Space group informed transformer for crystalline materials
  generation, 2024.

\bibitem[Castelli et~al.(2012)Castelli, Landis, Thygesen, Dahl, Chorkendorff,
  Jaramillo, and Jacobsen]{castelli2012new}
Castelli, I.~E., Landis, D.~D., Thygesen, K.~S., Dahl, S., Chorkendorff, I.,
  Jaramillo, T.~F., and Jacobsen, K.~W.
\newblock New cubic perovskites for one-and two-photon water splitting using
  the computational materials repository.
\newblock \emph{Energy \& Environmental Science}, 5\penalty0 (10):\penalty0
  9034--9043, 2012.

\bibitem[Chavel et~al.(1984)Chavel, Randol, and Dodziuk]{chavel1984eigenvalues}
Chavel, I., Randol, B., and Dodziuk, J. (eds.).
\newblock \emph{Eigenvalues in Riemannian Geometry}, volume 115 of \emph{Pure
  and Applied Mathematics}.
\newblock Elsevier, 1984.
\newblock \doi{https://doi.org/10.1016/S0079-8169(08)60810-7}.
\newblock URL
  \url{https://www.sciencedirect.com/science/article/pii/S0079816908608107}.

\bibitem[Cheetham \& Seshadri(2024)Cheetham and
  Seshadri]{cheetham2024artificial}
Cheetham, A.~K. and Seshadri, R.
\newblock Artificial intelligence driving materials discovery? perspective on
  the article: Scaling deep learning for materials discovery.
\newblock \emph{Chemistry of Materials}, 2024.

\bibitem[Chen \& Lipman(2024)Chen and Lipman]{chen2023riemannian}
Chen, R.~T. and Lipman, Y.
\newblock Riemannian flow matching on general geometries.
\newblock In \emph{The Twelfth International Conference on Learning
  Representations}, 2024.
\newblock URL \url{https://openreview.net/forum?id=g7ohDlTITL}.

\bibitem[Chen et~al.(2018)Chen, Rubanova, Bettencourt, and
  Duvenaud]{chen2018neural}
Chen, R.~T., Rubanova, Y., Bettencourt, J., and Duvenaud, D.~K.
\newblock Neural ordinary differential equations.
\newblock \emph{Advances in neural information processing systems}, 31, 2018.

\bibitem[Chen et~al.(2022)Chen, Zhang, and Hinton]{chen2022analog}
Chen, T., Zhang, R., and Hinton, G.
\newblock Analog bits: Generating discrete data using diffusion models with
  self-conditioning.
\newblock In \emph{The Eleventh International Conference on Learning
  Representations}, 2022.

\bibitem[Choubisa et~al.(2023)Choubisa, Todorovi{\'c}, Pina, Parmar, Li,
  Voznyy, Tamblyn, and Sargent]{choubisa2023interpretable}
Choubisa, H., Todorovi{\'c}, P., Pina, J.~M., Parmar, D.~H., Li, Z., Voznyy,
  O., Tamblyn, I., and Sargent, E.~H.
\newblock Interpretable discovery of semiconductors with machine learning.
\newblock \emph{npj Computational Materials}, 9\penalty0 (1):\penalty0 117,
  2023.

\bibitem[Court et~al.(2020)Court, Yildirim, Jain, and Cole]{court20203}
Court, C.~J., Yildirim, B., Jain, A., and Cole, J.~M.
\newblock 3-d inorganic crystal structure generation and property prediction
  via representation learning.
\newblock \emph{Journal of chemical information and modeling}, 60\penalty0
  (10):\penalty0 4518--4535, 2020.

\bibitem[Davies et~al.(2019)Davies, Butler, Jackson, Skelton, Morita, and
  Walsh]{davies2019smact}
Davies, D.~W., Butler, K.~T., Jackson, A.~J., Skelton, J.~M., Morita, K., and
  Walsh, A.
\newblock Smact: Semiconducting materials by analogy and chemical theory.
\newblock \emph{Journal of Open Source Software}, 4\penalty0 (38):\penalty0
  1361, 2019.

\bibitem[Deng et~al.(2023)Deng, Zhong, Jun, Riebesell, Han, Bartel, and
  Ceder]{deng2023chgnet}
Deng, B., Zhong, P., Jun, K., Riebesell, J., Han, K., Bartel, C.~J., and Ceder,
  G.
\newblock Chgnet as a pretrained universal neural network potential for
  charge-informed atomistic modelling.
\newblock \emph{Nature Machine Intelligence}, 5\penalty0 (9):\penalty0
  1031--1041, 2023.

\bibitem[Falorsi \& Forr{\'e}(2020)Falorsi and Forr{\'e}]{falorsi2020neural}
Falorsi, L. and Forr{\'e}, P.
\newblock Neural ordinary differential equations on manifolds.
\newblock \emph{arXiv preprint arXiv:2006.06663}, 2020.

\bibitem[Flam-Shepherd \& Aspuru-Guzik(2023)Flam-Shepherd and
  Aspuru-Guzik]{flam2023language}
Flam-Shepherd, D. and Aspuru-Guzik, A.
\newblock Language models can generate molecules, materials, and protein
  binding sites directly in three dimensions as xyz, cif, and pdb files.
\newblock \emph{arXiv preprint arXiv:2305.05708}, 2023.

\bibitem[Geiger \& Smidt(2022)Geiger and Smidt]{geiger2022e3nn}
Geiger, M. and Smidt, T.
\newblock e3nn: Euclidean neural networks.
\newblock \emph{arXiv preprint arXiv:2207.09453}, 2022.

\bibitem[Gemici et~al.(2016)Gemici, Rezende, and
  Mohamed]{gemici2016normalizing}
Gemici, M.~C., Rezende, D., and Mohamed, S.
\newblock Normalizing flows on riemannian manifolds.
\newblock \emph{arXiv preprint arXiv:1611.02304}, 2016.

\bibitem[Glass et~al.(2006)Glass, Oganov, and Hansen]{glass2006uspex}
Glass, C.~W., Oganov, A.~R., and Hansen, N.
\newblock Uspex—evolutionary crystal structure prediction.
\newblock \emph{Computer physics communications}, 175\penalty0
  (11-12):\penalty0 713--720, 2006.

\bibitem[Grosse-Kunstleve et~al.(2004)Grosse-Kunstleve, Sauter, and
  Adams]{grosse2004numerically}
Grosse-Kunstleve, R.~W., Sauter, N.~K., and Adams, P.~D.
\newblock Numerically stable algorithms for the computation of reduced unit
  cells.
\newblock \emph{Acta Crystallographica Section A: Foundations of
  Crystallography}, 60\penalty0 (1):\penalty0 1--6, 2004.

\bibitem[Gruver et~al.(2024)Gruver, Sriram, Madotto, Wilson, Zitnick, and
  Ulissi]{gruver2024fine}
Gruver, N., Sriram, A., Madotto, A., Wilson, A.~G., Zitnick, C.~L., and Ulissi,
  Z.
\newblock Fine-tuned language models generate stable inorganic materials as
  text.
\newblock \emph{arXiv preprint arXiv:2402.04379}, 2024.

\bibitem[Ho et~al.(2020)Ho, Jain, and Abbeel]{ho2020denoising}
Ho, J., Jain, A., and Abbeel, P.
\newblock Denoising diffusion probabilistic models.
\newblock \emph{Advances in Neural Information Processing Systems},
  33:\penalty0 6840--6851, 2020.

\bibitem[Hoogeboom et~al.(2022)Hoogeboom, Satorras, Vignac, and
  Welling]{hoogeboom2022equivariant}
Hoogeboom, E., Satorras, V.~G., Vignac, C., and Welling, M.
\newblock Equivariant diffusion for molecule generation in 3d.
\newblock In \emph{International Conference on Machine Learning}, pp.\
  8867--8887. PMLR, 2022.

\bibitem[Hu et~al.(2023)Hu, Liu, Zhang, and Yan]{hu2023machine}
Hu, E., Liu, C., Zhang, W., and Yan, Q.
\newblock Machine learning assisted understanding and discovery of co2
  reduction reaction electrocatalyst.
\newblock \emph{The Journal of Physical Chemistry C}, 127\penalty0
  (2):\penalty0 882--893, 2023.

\bibitem[Huang et~al.(2022)Huang, Aghajohari, Bose, Panangaden, and
  Courville]{huang2022riemannian}
Huang, C.-W., Aghajohari, M., Bose, J., Panangaden, P., and Courville, A.
\newblock Riemannian diffusion models.
\newblock In Oh, A.~H., Agarwal, A., Belgrave, D., and Cho, K. (eds.),
  \emph{Advances in Neural Information Processing Systems}, 2022.

\bibitem[Jain et~al.(2013)Jain, Ong, Hautier, Chen, Richards, Dacek, Cholia,
  Gunter, Skinner, Ceder, and Persson]{jain2013materials}
Jain, A., Ong, S.~P., Hautier, G., Chen, W., Richards, W.~D., Dacek, S.,
  Cholia, S., Gunter, D., Skinner, D., Ceder, G., and Persson, K.~A.
\newblock {The Materials Project: A materials genome approach to accelerating
  materials innovation}.
\newblock \emph{APL Materials}, 1\penalty0 (1):\penalty0 011002, 07 2013.
\newblock ISSN 2166-532X.
\newblock \doi{10.1063/1.4812323}.
\newblock URL \url{https://doi.org/10.1063/1.4812323}.

\bibitem[Jiao et~al.(2023)Jiao, Huang, Lin, Han, Chen, Lu, and
  Liu]{jiao2023crystal}
Jiao, R., Huang, W., Lin, P., Han, J., Chen, P., Lu, Y., and Liu, Y.
\newblock Crystal structure prediction by joint equivariant diffusion.
\newblock \emph{arXiv preprint arXiv:2309.04475}, 2023.

\bibitem[Jiao et~al.(2024)Jiao, Huang, Liu, Zhao, and Liu]{jiao2024space}
Jiao, R., Huang, W., Liu, Y., Zhao, D., and Liu, Y.
\newblock Space group constrained crystal generation, 2024.

\bibitem[K{\"o}hler et~al.(2020)K{\"o}hler, Klein, and
  No{\'e}]{kohler2020equivariant}
K{\"o}hler, J., Klein, L., and No{\'e}, F.
\newblock Equivariant flows: exact likelihood generative learning for symmetric
  densities.
\newblock In \emph{International conference on machine learning}, pp.\
  5361--5370. PMLR, 2020.

\bibitem[Kohn \& Sham(1965)Kohn and Sham]{kohn1965self}
Kohn, W. and Sham, L.~J.
\newblock Self-consistent equations including exchange and correlation effects.
\newblock \emph{Physical review}, 140\penalty0 (4A):\penalty0 A1133, 1965.

\bibitem[Kondor \& Trivedi(2018)Kondor and Trivedi]{kondor2018generalization}
Kondor, R. and Trivedi, S.
\newblock On the generalization of equivariance and convolution in neural
  networks to the action of compact groups.
\newblock In \emph{International Conference on Machine Learning}, pp.\
  2747--2755. PMLR, 2018.

\bibitem[Kresse \& Furthm{\"u}ller(1996)Kresse and
  Furthm{\"u}ller]{kresse1996efficient}
Kresse, G. and Furthm{\"u}ller, J.
\newblock Efficient iterative schemes for ab initio total-energy calculations
  using a plane-wave basis set.
\newblock \emph{Physical review B}, 54\penalty0 (16):\penalty0 11169, 1996.

\bibitem[Liao et~al.(2023)Liao, Wood, Das, and Smidt]{liao2023equiformerv2}
Liao, Y.-L., Wood, B., Das, A., and Smidt, T.
\newblock Equiformerv2: Improved equivariant transformer for scaling to
  higher-degree representations.
\newblock \emph{arXiv preprint arXiv:2306.12059}, 2023.

\bibitem[Ling(2022)]{ling2022review}
Ling, C.
\newblock A review of the recent progress in battery informatics.
\newblock \emph{npj Computational Materials}, 8\penalty0 (1):\penalty0 33,
  2022.

\bibitem[Lipman et~al.(2022)Lipman, Chen, Ben-Hamu, Nickel, and
  Le]{lipman2022flow}
Lipman, Y., Chen, R.~T., Ben-Hamu, H., Nickel, M., and Le, M.
\newblock Flow matching for generative modeling.
\newblock In \emph{The Eleventh International Conference on Learning
  Representations}, 2022.

\bibitem[Loshchilov \& Hutter(2018)Loshchilov and
  Hutter]{loshchilov2018decoupled}
Loshchilov, I. and Hutter, F.
\newblock Decoupled weight decay regularization.
\newblock In \emph{International Conference on Learning Representations}, 2018.

\bibitem[Mathieu \& Nickel(2020)Mathieu and Nickel]{mathieu2020riemannian}
Mathieu, E. and Nickel, M.
\newblock Riemannian continuous normalizing flows.
\newblock \emph{Advances in Neural Information Processing Systems},
  33:\penalty0 2503--2515, 2020.

\bibitem[Merchant et~al.(2023)Merchant, Batzner, Schoenholz, Aykol, Cheon, and
  Cubuk]{merchant2023scaling}
Merchant, A., Batzner, S., Schoenholz, S.~S., Aykol, M., Cheon, G., and Cubuk,
  E.~D.
\newblock Scaling deep learning for materials discovery.
\newblock \emph{Nature}, pp.\  1--6, 2023.

\bibitem[Miller et~al.(2020)Miller, Geiger, Smidt, and
  No{\'e}]{miller2020relevance}
Miller, B.~K., Geiger, M., Smidt, T.~E., and No{\'e}, F.
\newblock Relevance of rotationally equivariant convolutions for predicting
  molecular properties.
\newblock \emph{arXiv preprint arXiv:2008.08461}, 2020.

\bibitem[Nouira et~al.(2018)Nouira, Sokolovska, and
  Crivello]{nouira2018crystalgan}
Nouira, A., Sokolovska, N., and Crivello, J.-C.
\newblock Crystalgan: learning to discover crystallographic structures with
  generative adversarial networks.
\newblock \emph{arXiv preprint arXiv:1810.11203}, 2018.

\bibitem[Ong et~al.(2013)Ong, Richards, Jain, Hautier, Kocher, Cholia, Gunter,
  Chevrier, Persson, and Ceder]{ong2013python}
Ong, S.~P., Richards, W.~D., Jain, A., Hautier, G., Kocher, M., Cholia, S.,
  Gunter, D., Chevrier, V.~L., Persson, K.~A., and Ceder, G.
\newblock Python materials genomics (pymatgen): A robust, open-source python
  library for materials analysis.
\newblock \emph{Computational Materials Science}, 68:\penalty0 314--319, 2013.

\bibitem[Passaro \& Zitnick(2023)Passaro and Zitnick]{passaro2023reducing}
Passaro, S. and Zitnick, C.~L.
\newblock Reducing so (3) convolutions to so (2) for efficient equivariant
  gnns.
\newblock \emph{arXiv preprint arXiv:2302.03655}, 2023.

\bibitem[Perdew et~al.(1996)Perdew, Burke, and
  Ernzerhof]{perdew1996generalized}
Perdew, J.~P., Burke, K., and Ernzerhof, M.
\newblock Generalized gradient approximation made simple.
\newblock \emph{Physical review letters}, 77\penalty0 (18):\penalty0 3865,
  1996.

\bibitem[Pickard(2020)]{pickard2020cabon}
Pickard, C.~J.
\newblock Airss data for carbon at 10gpa and the c+n+h+o system at 1gpa.
\newblock \emph{Materials Cloud Archive}, 2020.0026/v1, 2020.
\newblock \doi{10.24435/materialscloud:2020.0026/v1}.

\bibitem[Pickard \& Needs(2011)Pickard and Needs]{pickard2011ab}
Pickard, C.~J. and Needs, R.
\newblock Ab initio random structure searching.
\newblock \emph{Journal of Physics: Condensed Matter}, 23\penalty0
  (5):\penalty0 053201, 2011.

\bibitem[Potyrailo et~al.(2011)Potyrailo, Rajan, Stoewe, Takeuchi, Chisholm,
  and Lam]{potyrailo2011combinatorial}
Potyrailo, R., Rajan, K., Stoewe, K., Takeuchi, I., Chisholm, B., and Lam, H.
\newblock Combinatorial and high-throughput screening of materials libraries:
  review of state of the art.
\newblock \emph{ACS combinatorial science}, 13\penalty0 (6):\penalty0 579--633,
  2011.

\bibitem[Riebesell(2024)]{riebesell2024convexhull}
Riebesell, J.
\newblock {Matbench Discovery v1.0.0}.
\newblock \emph{figshare}, 1 2024.
\newblock \doi{10.6084/m9.figshare.22715158.v12}.
\newblock URL
  \url{https://figshare.com/articles/dataset/Matbench_Discovery_v1_0_0/22715158}.

\bibitem[Riebesell et~al.(2023{\natexlab{a}})Riebesell, Goodall, Benner,
  Chiang, Lee, Jain, and Persson]{riebesell2023matbenchsoftware}
Riebesell, J., Goodall, R., Benner, P., Chiang, Y., Lee, A., Jain, A., and
  Persson, K.
\newblock {Matbench Discovery}, August 2023{\natexlab{a}}.
\newblock URL \url{https://github.com/janosh/matbench-discovery}.

\bibitem[Riebesell et~al.(2023{\natexlab{b}})Riebesell, Goodall, Jain, Benner,
  Persson, and Lee]{riebesell2023matbench}
Riebesell, J., Goodall, R.~E., Jain, A., Benner, P., Persson, K.~A., and Lee,
  A.~A.
\newblock Matbench discovery an evaluation framework for machine learning
  crystal stability prediction.
\newblock \emph{arXiv preprint arXiv:2308.14920}, 2023{\natexlab{b}}.

\bibitem[Satorras et~al.(2021)Satorras, Hoogeboom, and Welling]{satorras2021n}
Satorras, V.~G., Hoogeboom, E., and Welling, M.
\newblock E (n) equivariant graph neural networks.
\newblock In \emph{International conference on machine learning}, pp.\
  9323--9332. PMLR, 2021.

\bibitem[Schmidt et~al.(2022)Schmidt, Hoffmann, Wang, Borlido, Carri{\c{c}}o,
  Cerqueira, Botti, and Marques]{schmidt2022large}
Schmidt, J., Hoffmann, N., Wang, H.-C., Borlido, P., Carri{\c{c}}o, P.~J.,
  Cerqueira, T.~F., Botti, S., and Marques, M.~A.
\newblock Large-scale machine-learning-assisted exploration of the whole
  materials space.
\newblock \emph{arXiv preprint arXiv:2210.00579}, 2022.

\bibitem[Shaul et~al.(2023)Shaul, Chen, Nickel, Le, and
  Lipman]{shaul2023kinetic}
Shaul, N., Chen, R.~T., Nickel, M., Le, M., and Lipman, Y.
\newblock On kinetic optimal probability paths for generative models.
\newblock In \emph{International Conference on Machine Learning}, pp.\
  30883--30907. PMLR, 2023.

\bibitem[Sohl-Dickstein et~al.(2015)Sohl-Dickstein, Weiss, Maheswaranathan, and
  Ganguli]{sohl2015deep}
Sohl-Dickstein, J., Weiss, E., Maheswaranathan, N., and Ganguli, S.
\newblock Deep unsupervised learning using nonequilibrium thermodynamics.
\newblock In \emph{International Conference on Machine Learning}, pp.\
  2256--2265. PMLR, 2015.

\bibitem[Song et~al.(2020)Song, Sohl-Dickstein, Kingma, Kumar, Ermon, and
  Poole]{song2020score}
Song, Y., Sohl-Dickstein, J., Kingma, D.~P., Kumar, A., Ermon, S., and Poole,
  B.
\newblock Score-based generative modeling through stochastic differential
  equations.
\newblock \emph{arXiv preprint arXiv:2011.13456}, 2020.

\bibitem[Thomas et~al.(2018)Thomas, Smidt, Kearnes, Yang, Li, Kohlhoff, and
  Riley]{thomas2018tensor}
Thomas, N., Smidt, T., Kearnes, S., Yang, L., Li, L., Kohlhoff, K., and Riley,
  P.
\newblock Tensor field networks: Rotation-and translation-equivariant neural
  networks for 3d point clouds.
\newblock \emph{arXiv preprint arXiv:1802.08219}, 2018.

\bibitem[Wang et~al.(2021)Wang, Botti, and Marques]{wang2021predicting}
Wang, H.-C., Botti, S., and Marques, M.~A.
\newblock Predicting stable crystalline compounds using chemical similarity.
\newblock \emph{npj Computational Materials}, 7\penalty0 (1):\penalty0 12,
  2021.

\bibitem[Ward et~al.(2016)Ward, Agrawal, Choudhary, and
  Wolverton]{ward2016general}
Ward, L., Agrawal, A., Choudhary, A., and Wolverton, C.
\newblock A general-purpose machine learning framework for predicting
  properties of inorganic materials.
\newblock \emph{npj Computational Materials}, 2\penalty0 (1):\penalty0 1--7,
  2016.

\bibitem[Weiler et~al.(2021)Weiler, Forr{\'e}, Verlinde, and
  Welling]{weiler2021coordinate}
Weiler, M., Forr{\'e}, P., Verlinde, E., and Welling, M.
\newblock Coordinate independent convolutional networks--isometry and gauge
  equivariant convolutions on riemannian manifolds.
\newblock \emph{arXiv preprint arXiv:2106.06020}, 2021.

\bibitem[Wirnsberger et~al.(2022)Wirnsberger, Papamakarios, Ibarz,
  Racani{\`e}re, Ballard, Pritzel, and Blundell]{wirnsberger2022normalizing}
Wirnsberger, P., Papamakarios, G., Ibarz, B., Racani{\`e}re, S., Ballard,
  A.~J., Pritzel, A., and Blundell, C.
\newblock Normalizing flows for atomic solids.
\newblock \emph{Machine Learning: Science and Technology}, 3\penalty0
  (2):\penalty0 025009, 2022.

\bibitem[Xie et~al.(2021)Xie, Fu, Ganea, Barzilay, and
  Jaakkola]{xie2021crystal}
Xie, T., Fu, X., Ganea, O.-E., Barzilay, R., and Jaakkola, T.~S.
\newblock Crystal diffusion variational autoencoder for periodic material
  generation.
\newblock In \emph{International Conference on Learning Representations}, 2021.

\bibitem[Yang et~al.(2023)Yang, Cho, Merchant, Abbeel, Schuurmans, Mordatch,
  and Cubuk]{yang2023scalable}
Yang, M., Cho, K., Merchant, A., Abbeel, P., Schuurmans, D., Mordatch, I., and
  Cubuk, E.~D.
\newblock Scalable diffusion for materials generation.
\newblock \emph{arXiv preprint arXiv:2311.09235}, 2023.

\bibitem[Yang et~al.(2021)Yang, Siriwardane, Dong, Li, and Hu]{yang2021crystal}
Yang, W., Siriwardane, E. M.~D., Dong, R., Li, Y., and Hu, J.
\newblock Crystal structure prediction of materials with high symmetry using
  differential evolution.
\newblock \emph{Journal of Physics: Condensed Matter}, 33\penalty0
  (45):\penalty0 455902, 2021.

\bibitem[Yim et~al.(2023)Yim, Campbell, Foong, Gastegger, Jim{\'e}nez-Luna,
  Lewis, Satorras, Veeling, Barzilay, Jaakkola, et~al.]{yim2023fast}
Yim, J., Campbell, A., Foong, A.~Y., Gastegger, M., Jim{\'e}nez-Luna, J.,
  Lewis, S., Satorras, V.~G., Veeling, B.~S., Barzilay, R., Jaakkola, T.,
  et~al.
\newblock Fast protein backbone generation with se (3) flow matching.
\newblock \emph{arXiv preprint arXiv:2310.05297}, 2023.

\bibitem[Zeni et~al.(2023)Zeni, Pinsler, Z{\"u}gner, Fowler, Horton, Fu,
  Shysheya, Crabb{\'e}, Sun, Smith, et~al.]{zeni2023mattergen}
Zeni, C., Pinsler, R., Z{\"u}gner, D., Fowler, A., Horton, M., Fu, X.,
  Shysheya, S., Crabb{\'e}, J., Sun, L., Smith, J., et~al.
\newblock Mattergen: a generative model for inorganic materials design.
\newblock \emph{arXiv preprint arXiv:2312.03687}, 2023.

\bibitem[Zimmermann \& Jain(2020)Zimmermann and Jain]{zimmermann2020local}
Zimmermann, N.~E. and Jain, A.
\newblock Local structure order parameters and site fingerprints for
  quantification of coordination environment and crystal structure similarity.
\newblock \emph{RSC advances}, 10\penalty0 (10):\penalty0 6063--6081, 2020.

\end{thebibliography}
\bibliographystyle{icml2024}

\newpage
\appendix
\onecolumn
\section{Preliminaries Continued}
\label{appendix:preliminaries}

\subsection{Datasets}
We consider two realistic datasets and two unit test datasets. The first two are realistic and the second two are unit tests. All datasets are divided into 60\% training data, 20\% validation data, 20\% test data. We use the same splits as \citet{xie2021crystal} and \citet{jiao2023crystal}.

\paragraph{Materials Project-20 \cite{xie2021crystal}}

Also known as \emph{MP-20}. Contains 45231 samples. Contains all materials with a maximum of 20 atoms per unit cell and within 0.08 eV/atom in the Materials Project database \cite{jain2013materials} from around July 2021. Materials containing radioactive atoms are removed.

\paragraph{Materials Project Time Splits-52 \cite{baird2024mpts}}
Also known as \emph{MPTS-52}. Contains 40476 samples. Uses similar criteria as MP-20, but allows materials with atoms up to 52 in a single unit cell and no elemental filtering is applied. Furthermore, the train, validation, and test splits are organized in chronological order. Therefore, the oldest materials are in the training set and the newest ones are in the test set. Note: this dataset has fewer samples than MP-20 because some materials are entered into the Materials Project database without first publication timestamp information. Those materials are omitted from the dataset.

\paragraph{Perovskite-5 \cite{castelli2012new}}
Also known as \emph{perov} or \emph{perov-5}. Contains 18928 samples. All materials have five atoms per unit cell located at the same fractional coordinate values and lattice angles are fixed. Only the lattice lengths and atomic types change.

\paragraph{Carbon-24 \cite{pickard2020cabon}}
Also known as \emph{carbon}. Contains 10153 samples. Each material contains only carbon atoms, but the other variables are not fixed. This leads to a challenging CSP problem because there are typically multiple geometries for every n-atom set of carbon atoms. This is reflected in depressed match rate scores.

\subsection{Proxy metrics}
\paragraph{Crystal Structure Prediction}
Following \citet{jiao2023crystal}, we sampled the CSP model using held out structures as conditioning and measured the \emph{match rate} and \emph{Root-Mean-Square Error (RMSE)}, according to the output of \texttt{StructureMatcher} \citet{ong2013python} with settings $stol=0.5, angle\_tol=10, ltol=0.3$. 
Match rate is the number of generated structures that \texttt{StructureMatcher} find are within tolerances defined above divided by the total number of held-out structures. RMSE is computed when the held-out and generated structures match (otherwise it does not enter the reported statistics), then normalized by $(V/N)^{1/3}$ as is standard. Unlike DiffCSP, we did not compute multi-sample statistics given the same input composition.

\paragraph{De novo generation}
A composition is structurally valid when all pairwise distance between atoms are greater than 0.5 \AA. A crystal is compositionally valid when a simple heuristic system, SMACT \cite{davies2019smact}, determines that the crystal would be charge neutral. Coverage for both COV-R (recall) and COV-P (precision) are standard Recall \& Precision metrics computed after on thresholding pairwise distances between 1,000 samples that are both compositionally and structurally valid, and have been featurized by CrystalNN structural fingerprints~\citep{zimmermann2020local} and the normalized Magpie compositional fingerprints~\citep{ward2016general}.

We also compute two Wasserstein distances on computed properties of crystal samples from the test set and our generated structures. Namely, $d_\rho$ and $d_{N_{el}}$, corresponding to distances between the atomic density: number of atoms divided by unit cell volume and $N_{el}$ which is the number of unique elements in the unit cell, aka $N$-ary.

\subsection{Riemannian Manifolds}

Since $\Ccal$ is a product of Riemannian manifolds, it has a natural metric: For any $(\abold, \fbold, \lbold) \in \Ccal$, the tangent space $(\Acal \times \Fcal \times \Lcal)_{(\abold, \fbold, \lbold)}$ is canonically isomorphic to the direct sum $\Acal_{\abold} \oplus \Fcal_{\fbold} \oplus \Lcal_{\lbold}$. For vectors $\xi, \eta,  \in \Acal_{\abold}$; $\zeta, \chi\in \Fcal_{\fbold}$; and $\gamma, \omega \in \Lcal_{\lbold}$ we define the inner product of $\xi \oplus \zeta \oplus \gamma$ and $\eta \oplus \chi\oplus \omega$ by $\langle\xi \oplus \zeta \oplus \gamma, \eta \oplus \chi \oplus \omega \rangle_{(\abold, \fbold, \lbold)} \coloneqq \lambda_{a} \langle\xi, \eta\rangle_{\abold} + \lambda_{f} \langle \zeta, \xi \rangle_{\fbold} + \lambda_{l} \langle \gamma , \omega \rangle_{\lbold}$, where the subscripts indicate the tangent space in which the different inner products are calculated and $\lambda \coloneqq (\lambda_{a}, \lambda_{f}, \lambda_{l}) \in \Rbb^{3}$ is a hyperparameter. In particular, $\Acal_{\abold} \oplus\{0\} \oplus \{0\}$ is orthogonal to $\{0\} \oplus \Fcal_{\fbold} \oplus \{0\}$ and $\{0\} \oplus \{0\} \oplus \Lcal_{\lbold}$. The associated Riemannian measure on $\Acal \times \Fcal \times \Lcal$ is the product measure determined by $d V_{\Acal}$, $d V_{\Fcal}$, and $d V_{\Lcal}$ where $d V_\square$ denotes the Lebesgue measure on space $\square$
\citep{chavel1984eigenvalues}. Since our measure is a product measure, we may define the base probability measures of each space by densities absolutely continuous with respect to their respective Lebesgue measure.

\subsection{Specifics for De Novo Generation}
\paragraph{Determining number of atoms in the unit cell}
Above, we describe De Novo Generation via a distribution $p(\abold, \fbold, \lbold)$; however, this omits an important variable: $n$ the number of atoms in the unit cell. This distribution carries an implicit conditional on the number of atoms, namely $p(\abold, \fbold, \lbold \mid n)$. In other words, we have assumed that $n$ is known beforehand. However, we are interested in generating materials with a variable number of atoms. To solve this problem we follow the method of \citet{hoogeboom2022equivariant} and first sample $n \sim p(n)$ from the empirical distribution of the training set.

\paragraph{Methodology for identifying Stable, Unique, and Novel (S.U.N.) materials}
Our goal in DNG is to generate stable, unique and novel materials. In that effort, we generated samples from FlowMM, prerelaxed them using CHGNet, and finally relaxed them using density functional theory. Our method for determining whether a material is S.U.N. is as follows: 

(S) We determine the stability of our relaxed structures against the \texttt{Matbench Discovery} \cite{riebesell2023matbenchsoftware, riebesell2023matbench} convex hull, compiled from the Materials Project \cite{jain2013materials}, marked as 2nd Feburary 2023. (Although our training data comes from an earlier version of the database, we can still estimate the performance of the models using a later version of the convex hull. In this situation, it becomes more difficult to generate novel structures since those proposed structures may have been added to the database between 2021 and 2023.)

(N) We then take our stable generated structures and search the training data for any structure which contains the same set of elements. We ignore the frequency of the elements during this search, in order to catch similar materials with differently defined unit cells. We do a pairwise comparison between the generated structure and all ``element-matching'' examples from the training set using \texttt{StructureMatcher} \cite{ong2013python} with default settings. If there is no match, we consider that structure novel. 

(U) Finally, we take all stable and novel structures, then use \texttt{StructureMatcher} to pairwise compare those structures with themselves. We collect all pairwise matches and group them into ``equivalent'' structures. This group counts as only one structure for the purpose of S.U.N. computations, thereby enforcing uniqueness.

We want to emphasize that this is not a perfect system. \texttt{StructureMatcher} may fail to detect a match, or falsely detect one, due to the default settings of its threshold. Furthermore, without careful application beyond the default settings, \texttt{StructureMatcher} does not tell us about chemical function and may not yield matches for materials with extremely similar chemical properties. This could inflate the estimated number of S.U.N. materials \cite{cheetham2024artificial}. Additionally, \texttt{StructureMatcher} does not define an equivalence relation since it does satisfy the reflexivity property. We treat it like one here anyway since it holds approximately.

\paragraph{Stability metrics explained}
We are interested in several stability metrics:
\begin{align}
    Stability \; Rate &\coloneqq \frac{N_{\text{stable}}}{N_{\text{gen}}} \\
    Cost &\coloneqq \frac{N_{\text{int. steps}}}{Stability \; Rate} \\
	S.U.N. \; Rate &\coloneqq \frac{N_{\text{S.U.N.}}}{N_{\text{gen}}} \\
	S.U.N. \; Cost &\coloneqq \frac{N_{\text{int. steps}}}{S.U.N. \; Rate}
\end{align}
where $N_{\text{gen}}$ is the number of generated samples, $N_{\text{stable}}$ is the number of generated samples which are stable; $N_{\text{S.U.N.}}$ are the number of generated samples which are stable, unique, and novel; and $N_{\text{int. steps}}$ is the number of integration steps to produce a generated sample. By definition $N_{\text{S.U.N.}} \leq N_{\text{stable}} \leq N_{\text{gen}}$.

\subsection{Symmetry}
We discuss invariances to symmetry groups for crystal structures. We are interested in estimating a density with invariances to permutation, translation, and rotation as formalized in \eqref{eqn:permutation_invariance}, \eqref{eqn:translation_invariance}, and \eqref{eqn:rotation_invariance}. We visualize those symmetries in Figure~\ref{fig:symmetry}.
\begin{figure}[ht]
    \begin{center}
    \centerline{\includegraphics[width=0.7\columnwidth]{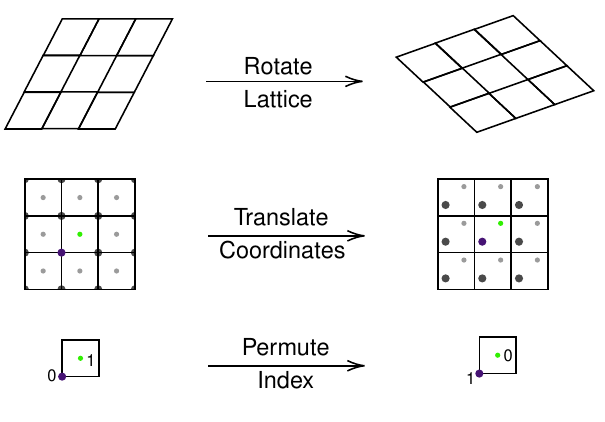}}
    \caption{Three symmetry actions are shown above; all of these actions would alter only the representation of the crystal, while leaving its chemical properties intact. (top) Rotation of a lattice formed by a unit cell. (mid) Translation of fractional coordinates within a unit cell. (bot) Permutation of atomic index. Since these images are two-dimensional, they do not capture the symmetry of a three-dimensional crystal. Furthermore, there are additional symmetries that are not represented in these pictures.}
    \label{fig:symmetry}
    \end{center}
    \vspace{-2em}
\end{figure}

There are additional symmetries for crystals that we did not explicitly model in FlowMM. Those are \emph{periodic cell choice invariance}: where the unit cell is skewed by $A$ with $\det A = 1$ and $A \in \Zbb^{3 \times 3}$ and the fractional coordinates are anti-skewed by $A^{-1}$, and \emph{supercell invariance} where the unit cell is grown to encompass another neighboring ``block'' and all of the atoms inside. These are discussed in more detail and visualized in \citet{zeni2023mattergen}.

\subsection{Riemannian Flow Matching visualization}
Since flow matching can be rather formal, and perhaps unintuitive when written symbolically, we draw cartoon representations of the regression target from the conditional vector field $u_t(\cdot \mid \cdot_{1})$ for the fractional coordinates $\fbold$ and lattice parameters $\lbold$ in Figure~\ref{fig:flow_matching_flat_torus} and Figure~\ref{fig:lengths_angles}, respectively.
In both cases, we also represent all necessary components to define the regression target namely the sample from the base distribution $\square_0$, the sample from the target distribution $\square_1$, the conditional path connecting them on the correct manifold, the point at time $t$ along the path where the conditional vector field is evaluated $\square_t$, and the conditional vector itself $u_t(\square_t \mid \square_{1})$ where $\square$ represents the relevant variable. We do not show $\abold$ because it occurs in Euclidean space and behaves like typical flow matching during training.
\begin{figure}[ht]
    \begin{center}
    \centerline{\includegraphics[width=0.5\columnwidth]{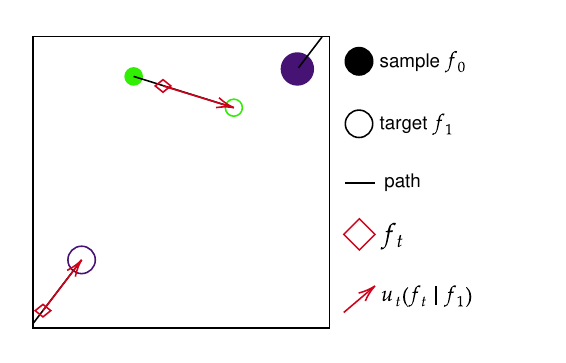}}
    \caption{We visualize the necessary components to express a hypothetical Riemannian Flow Matching regression target on a single data point for the fractional coordinates. The sample $f_0$ is drawn from the base uniform distribution for both points. The target $f_1$ is from the database of crystals. The path is drawn between the sample and the target following the geodesic path, i.e. wrapping around the boundary. The point $f_t$ along the path at time $t$ is indicated with a diamond. The conditional vector $u_t(f_t \mid f_1)$ at time $t$ is indicated as a vector. This vector is the regression target in Riemannian Flow Matching.}%
    \label{fig:flow_matching_flat_torus}
    \end{center}
    \vspace{-2em}
\end{figure}
\begin{figure}[ht]
    \begin{center}
    \centerline{\includegraphics[width=0.9\columnwidth]{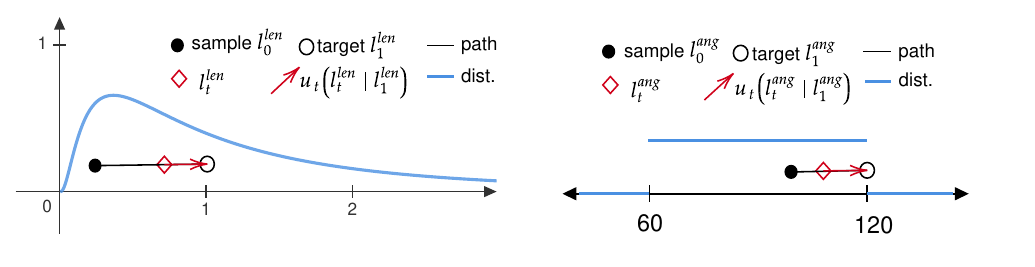}}
    \caption{We visualize the necessary components to express a hypothetical Flow Matching regression target on a single data point for the lattice parameters. Specifically, the length parameters are shown on the left and the angle parameters are on the right.
    The sample, target, path, point-along-path, and vector all follow the description in Figure~\ref{fig:flow_matching_flat_torus}, but are adapted for lattice parameters. Additionally, we display the base distribution as a blue line on this plot to indicate where samples $l_0^{len}$ and $l_0^{ang}$ can appear. 
    Note: This visualization occurs in so-called constrained space for $l^{ang}$; however, our proposed method does flow matching in the unconstrained space of $l^{ang}$ to avoid the boundaries of the $l^{ang}$ distribution. In this way, this figure visualises the challenges of doing precise flow matching in constrained space and the corresponding difficulty (and lower performance) of the ablated model. Recall that our transformation sends $60 \to -\infty$ and $120 \to \infty$.}
    \label{fig:lengths_angles}
    \end{center}
    \vspace{-2em}
\end{figure}

\subsection{Details about Density Functional Theory calculations}
For the stability metrics, we applied the Vienna ab initio simulation package (VASP) \cite{kresse1996efficient} to compute relaxed geometries and ground state energies at a temperature of 0 K and pressure of 0 atm. 
We used the default settings from the Materials Project \cite{jain2013materials} known as the \texttt{MPRelaxSet} with the PBE functional \cite{perdew1996generalized} and Hubbard U corrections. These correspond with the settings that our prerelaxation network CHGNet \cite{deng2023chgnet} was trained on, so prerelaxation should reduce DFT energy, up to the fitting error in CHGNet.

\subsection{Limitations of quantifying a computational approach to materials discovery}
There are a number of important limitations when it comes to using and quantifying the performance of generative models for materials discovery. 

Fundamental limitations for all computational methods include, but are not limited to: (a) Energy and stability computations all occur at nonphysical zero temperature and pressure settings. (b) Our material representation is not realistic since it assumes complete homogeneity and an infinite crystal structure without disorder. (c) There is a fundamental inaccuracy in density functional theory itself due to the basis set, the energy functional, and computational cost limitations... (etc.)

Generative models learn to fit empirical distributions. We are interested in generating S.U.N. materials which are not in our empirical distribution. In an imprecise way, we expect that FlowMM will generated materials that exist as ``interpolations'' between existing structures; however, the most interesting and new structures are well outside the existing empirical distribution. We do not expect FlowMM to find these interesting and new structures since it is not trained to do that.

Additionally, one must consider that proposed materials can still be extremely implausible, despite satisfying our definition of stability, or count a new material according to \texttt{StructureMatcher}, but a domain expert would not agree. Further discussion of these issues can be found in the work by \citet{cheetham2024artificial}. Furthermore, our tests using \texttt{StructureMatcher} rely on it defining an equivalence relation between structures; however, it does not due to its \texttt{rtol} parameter which means the reflexive property does not always hold. (However it does hold approximately.)

Finally, we want to emphasize that although we believe our \emph{Cost} metrics to be a good faith attempt to compare models, of course number of integration steps is only one of many dimensions to evaluate the cost of generating a novel material. We did not include training or relaxation time in these computations, for example. (Training time was approximately the same across models and relaxation time is independent of the generation method. Although, relaxation can depend on the accuracy of a reconstructed/generated structure.)

\section{Further Results}
\label{appendix:results_continued}

\subsection{Crystal Structure Prediction (CSP)}

To better understand how the components of FlowMM affect the performance on the CSP task, we performed several ablation studies and included estimates of uncertainty. We focused on three aspects in particular: (a) Ablating velocity anti-annealing, (b) ablating our proposed base distribution and unconstrained transformation for lattice parameters $\lbold$, and (c) estimating uncertainty in the inference stage, \emph{not during training}, by rerunning reconstruction with varied random seeds. 

Velocity anti-annealing is an inference-time hyperparameter that controls how much to scale the velocity prediction from our learned vector field during integration, see \eqref{eqn:anti-annealing}. We choose to apply velocity anti-annealing to no variables, fractional coordinates $\fbold$, lattice parameters $\lbold$, or both variables $\fbold, \lbold$. This amounts to a comprehensive ablation of the method. In CSP, we found that it was generally beneficial to apply velocity anti-annealing to $\fbold$, but not $\lbold$. We note that FlowMM typically saw improved performance compared to competitors without the need to scale the velocity with anti-annealing. We report these results in Tables~\ref{table:ablation_unit_test} and~\ref{table:ablation_realistic}. We also reproduce the match rate as a function of integration steps using FlowMM without Inference Anti-Annealing in Figure~\ref{fig:match_rate_all_v1}.

We describe our bespoke parameterization of $\lbold$ in Section~\ref{sec:rfm4mat} including a custom base distribution and a transformation to unconstrained space. 
Our neural network representation (Appendix~\ref{appendix:neural_network}) does not depend on ``physical'' lattice parameters. Therefore, it is also possible to simply use a typical normal base distribution, without transforming to unconstrained space, and let flow matching take care of learning the target distributions without inductive bias. Note that the ablation of the base distribution for lattice parameters requires training another model. During inference, such a model can produce representations that do not correspond to a real crystal; we will simply consider those generations as having failed. We jointly ablate these inductive biases along with a velocity anti-annealing ablation. We find that no matter the velocity anti-annealing scheme, our lattice parameter inductive biases provide a siginifcant performance boost.

Finally, for each case and dataset, we reran the generation from the corresponding trained model with the same hyperparameters three times. This therefore indicates variation at inference time, but not variation during training. (Although these are newly trained models compared to what is reported in Table~\ref{table:conditional_metrics} for every set of hyperparameter settings.) See Table~\ref{table:ablation_unit_test} for the ablation result on the unit test datasets and see Table~\ref{table:ablation_realistic} for the realistic datasets. Some unit test datasets reported the same match rate across three reconstructions when $n=20$, that's what leads to $\pm 0.0$.

\begin{figure}[ht]
    \vskip 0.0in
    \begin{center}
    \centerline{\includegraphics[width=0.5\columnwidth]{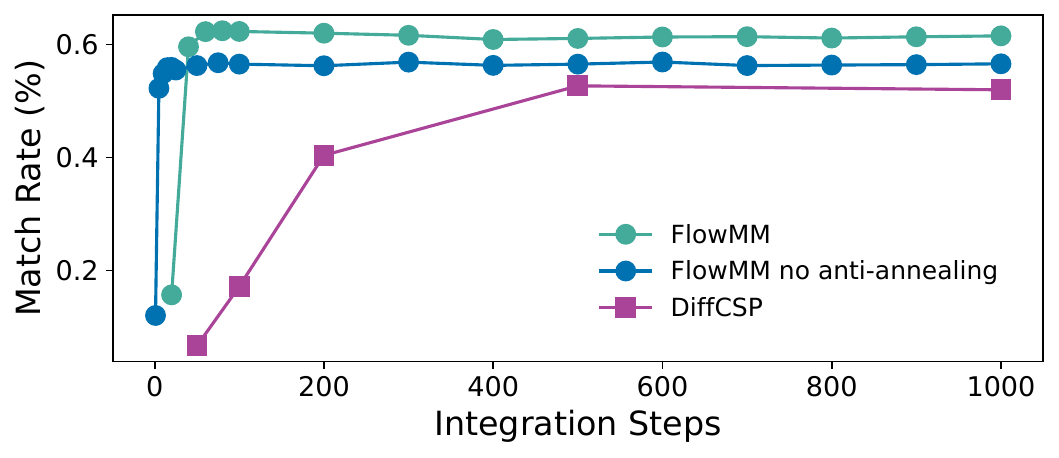}}
    \caption{Match rate as a function of number of integration steps on MP-20. FlowMM achieves a higher maximum match rate than DiffCSP overall without without Inference Anti-Annealing (on $\fbold$ not $\lbold$). Even without Inference Anti-Annealing, FlowMM outperforms DiffCSP at every number of integration steps.}
    \label{fig:match_rate_all_v1}
    \end{center}
    \vskip -0.4in
\end{figure}

\clearpage

\subsection{De Novo Generation (DNG)}
We present the distribution of generated stable crystals from CDVAE, DiffCSP, and FlowMM trained on MP-20 in Figure~\ref{fig:nary_stable_dist}. (For context, we include generations without regard to stability in Figure~\ref{fig:nary_all_dist}) The number of structures determined to be stable diminishes quickly as a function of $N$-ary, implying that models generating high $N$-ary materials do not relax to stable structures after density functional theory calculations.

\begin{figure}[htb]
    \centering
    \begin{subfigure}[b]{0.45\columnwidth}
        \centering
        \includegraphics[width=\columnwidth]{figures/nary_all_dist.pdf}
        \caption{Materials from MP-20 \& generative methods.}
        \label{fig:nary_all_dist}
    \end{subfigure}%
    \hspace{0.05\columnwidth}%
    \begin{subfigure}[b]{0.45\columnwidth}
        \centering
        \includegraphics[width=\columnwidth]{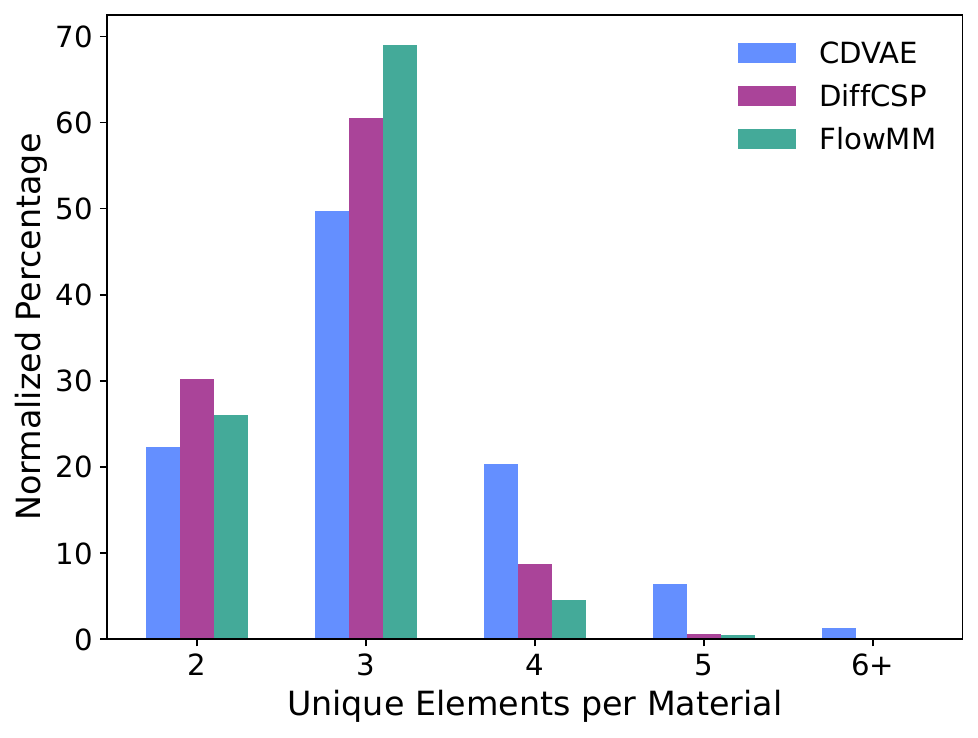}
        \caption{Stable materials from generative methods.}
        \label{fig:nary_stable_dist}
    \end{subfigure}%
    \caption{(a) Figure~\ref{fig:nary_dist} repeated for context. It represents a normalized histogram of the number of unique elements per crystal ($N$-ary) for MP-20 and the generative methods. (b) A normalized histogram of the number of unique elements per \emph{stable} crystal ($N$-ary) of structures generated by CDVAE, DiffCSP, and FlowMM. All structures are stable and were relaxed using density functional theory. Despite generating a large number of high $N$-ary structures, CDVAE and DiffCSP find relatively few stable ones after relaxation. The FlowMM columns correspond to generations with 1000 integration steps.}
    \label{fig:comparison}
\end{figure}

\begin{table*}
\centering
\caption{Results from ablation study on unit test datasets}
\label{table:ablation_unit_test}
\resizebox{0.95\textwidth}{!}{%
\begin{tabular}{llllllll}
\toprule
 &  &  &  & \multicolumn{2}{c}{Perov-5} & \multicolumn{2}{c}{Carbon-24} \\
 &  &  &  & Match Rate (\%) $\uparrow$ & RMSE $\downarrow$ & Match Rate (\%) $\uparrow$ & RMSE $\downarrow$ \\
\makecell{\# of\\Samples} & \makecell{Lattice $\lbold$\\Base Dist.} & \makecell{Anneal\\Coords} & \makecell{Anneal\\Lattice} &  &  &  &  \\
\midrule
\multirow[t]{8}{*}{1} & \multirow[t]{4}{*}{ablated} & \multirow[t]{2}{*}{False} & False & $52.62 \pm 1.06$ & $0.2822 \pm 0.0005$ & $15.70 \pm 0.85$ & $0.4262 \pm 0.0033$ \\
 &  &  & True & $0.00 \pm 0.00$ &  & $1.84 \pm 0.24$ & $0.4265 \pm 0.0041$ \\
\cline{3-8}
 &  & \multirow[t]{2}{*}{True} & False & $52.07 \pm 1.25$ & $0.2189 \pm 0.0062$ & $17.01 \pm 0.59$ & $0.4213 \pm 0.0024$ \\
 &  &  & True & $0.00 \pm 0.00$ &  & $1.71 \pm 0.12$ & $0.4229 \pm 0.0140$ \\
\cline{2-8} \cline{3-8}
 & \multirow[t]{4}{*}{proposed} & \multirow[t]{2}{*}{False} & False & $57.63 \pm 1.03$ & $0.2556 \pm 0.0039$ & $21.95 \pm 0.51$ & $0.4097 \pm 0.0017$ \\
 &  &  & True & $56.36 \pm 0.42$ & $0.1945 \pm 0.0020$ & $15.83 \pm 0.51$ & $0.3900 \pm 0.0017$ \\
\cline{3-8}
 &  & \multirow[t]{2}{*}{True} & False & $56.49 \pm 1.09$ & $0.1976 \pm 0.0025$ & $20.61 \pm 0.52$ & $0.3984 \pm 0.0027$ \\
 &  &  & True & $55.39 \pm 0.20$ & $0.1913 \pm 0.0036$ & $15.35 \pm 0.68$ & $0.3923 \pm 0.0055$ \\
\cline{1-8} \cline{2-8} \cline{3-8}
\multirow[t]{8}{*}{20} & \multirow[t]{4}{*}{ablated} & \multirow[t]{2}{*}{False} & False & $98.60 \pm 0.00$ & $0.0519 \pm 0.0007$ & $76.86 \pm 0.29$ & $0.3941 \pm 0.0008$ \\
 &  &  & True & $0.00 \pm 0.00$ &  & $26.88 \pm 1.12$ & $0.4208 \pm 0.0038$ \\
\cline{3-8}
 &  & \multirow[t]{2}{*}{True} & False & $98.60 \pm 0.00$ & $0.0372 \pm 0.0006$ & $77.93 \pm 0.23$ & $0.3870 \pm 0.0015$ \\
 &  &  & True & $0.00 \pm 0.00$ &  & $26.26 \pm 0.97$ & $0.4211 \pm 0.0026$ \\
\cline{2-8} \cline{3-8}
 & \multirow[t]{4}{*}{proposed} & \multirow[t]{2}{*}{False} & False & $98.60 \pm 0.00$ & $0.0428 \pm 0.0004$ & $79.85 \pm 0.36$ & $0.3599 \pm 0.0012$ \\
 &  &  & True & $98.60 \pm 0.00$ & $0.0334 \pm 0.0008$ & $75.85 \pm 1.31$ & $0.3349 \pm 0.0007$ \\
\cline{3-8}
 &  & \multirow[t]{2}{*}{True} & False & $98.60 \pm 0.00$ & $0.0328 \pm 0.0007$ & $84.15 \pm 0.54$ & $0.3301 \pm 0.0037$ \\
 &  &  & True & $98.60 \pm 0.00$ & $0.0331 \pm 0.0007$ & $76.08 \pm 0.16$ & $0.3376 \pm 0.0035$ \\
\cline{1-8} \cline{2-8} \cline{3-8}
\bottomrule
\end{tabular}%
}
\end{table*}

\begin{table*}
\centering
\caption{Results from ablation study on realistic datasets}
\label{table:ablation_realistic}
\resizebox{0.95\textwidth}{!}{%
\begin{tabular}{llllllll}
\toprule
 &  &  &  & \multicolumn{2}{c}{MP-20} & \multicolumn{2}{c}{MPTS-52} \\
 &  &  &  & Match Rate (\%) $\uparrow$ & RMSE $\downarrow$ & Match Rate (\%) $\uparrow$ & RMSE $\downarrow$ \\
\makecell{\# of\\Samples} & \makecell{Lattice $\lbold$\\Base Dist.} & \makecell{Anneal\\Coords} & \makecell{Anneal\\Lattice} &  &  &  &  \\
\midrule
\multirow[t]{8}{*}{1} & \multirow[t]{4}{*}{ablated} & \multirow[t]{2}{*}{False} & False & $43.21 \pm 0.47$ & $0.1812 \pm 0.0012$ & $4.04 \pm 0.21$ & $0.3490 \pm 0.0152$ \\
 &  &  & True & $0.00 \pm 0.00$ &  & $0.16 \pm 0.06$ & $0.4276 \pm 0.0306$ \\
\cline{3-8}
 &  & \multirow[t]{2}{*}{True} & False & $49.15 \pm 0.01$ & $0.0866 \pm 0.0012$ & $5.27 \pm 0.14$ & $0.2567 \pm 0.0091$ \\
 &  &  & True & $0.00 \pm 0.00$ &  & $0.15 \pm 0.03$ & $0.4255 \pm 0.0273$ \\
\cline{2-8} \cline{3-8}
 & \multirow[t]{4}{*}{proposed} & \multirow[t]{2}{*}{False} & False & $56.82 \pm 0.42$ & $0.1332 \pm 0.0016$ & $12.12 \pm 0.36$ & $0.2843 \pm 0.0056$ \\
 &  &  & True & $59.23 \pm 0.08$ & $0.0562 \pm 0.0018$ & $14.72 \pm 0.45$ & $0.1734 \pm 0.0020$ \\
\cline{3-8}
 &  & \multirow[t]{2}{*}{True} & False & $61.26 \pm 0.14$ & $0.0572 \pm 0.0014$ & $16.11 \pm 0.17$ & $0.1831 \pm 0.0021$ \\
 &  &  & True & $59.19 \pm 0.34$ & $0.0577 \pm 0.0008$ & $14.71 \pm 0.23$ & $0.1724 \pm 0.0012$ \\
\cline{1-8} \cline{2-8} \cline{3-8}
\multirow[t]{8}{*}{20} & \multirow[t]{4}{*}{ablated} & \multirow[t]{2}{*}{False} & False & $73.59 \pm 0.10$ & $0.1449 \pm 0.0006$ & $15.81 \pm 0.36$ & $0.3225 \pm 0.0025$ \\
 &  &  & True & $0.00 \pm 0.00$ &  & $1.56 \pm 0.03$ & $0.4298 \pm 0.0014$ \\
\cline{3-8}
 &  & \multirow[t]{2}{*}{True} & False & $75.68 \pm 0.20$ & $0.0791 \pm 0.0011$ & $20.63 \pm 0.24$ & $0.2581 \pm 0.0006$ \\
 &  &  & True & $0.00 \pm 0.00$ &  & $1.51 \pm 0.05$ & $0.4218 \pm 0.0084$ \\
\cline{2-8} \cline{3-8}
 & \multirow[t]{4}{*}{proposed} & \multirow[t]{2}{*}{False} & False & $76.55 \pm 0.09$ & $0.0834 \pm 0.0005$ & $28.99 \pm 0.04$ & $0.2445 \pm 0.0008$ \\
 &  &  & True & $70.07 \pm 0.04$ & $0.0472 \pm 0.0003$ & $28.73 \pm 0.25$ & $0.1655 \pm 0.0031$ \\
\cline{3-8}
 &  & \multirow[t]{2}{*}{True} & False & $75.81 \pm 0.07$ & $0.0479 \pm 0.0004$ & $34.05 \pm 0.04$ & $0.1813 \pm 0.0012$ \\
 &  &  & True & $70.00 \pm 0.11$ & $0.0474 \pm 0.0004$ & $28.95 \pm 0.09$ & $0.1666 \pm 0.0028$ \\
\cline{1-8} \cline{2-8} \cline{3-8}
\bottomrule
\end{tabular}%
}
\end{table*}

\clearpage

\section{Neural network}
\label{appendix:neural_network}
We employ a graph neural network from \citet{jiao2023crystal} that adapts EGNN \cite{satorras2021n} to fractional coordinates,
\begin{align}
    \label{eqn:mp0}
    \hbold^{i}_{(0)} &= \phi_{\hbold_{(0)}}(a^{i}) \\
    \label{eqn:mp1}
    \mbold^{ij}_{(s)} &= \varphi_m(\hbold^{i}_{(s-1)}, \hbold^{j}_{(s-1)}, \lbold, \text{SinusoidalEmbedding}(f^{j} - f^{i})), \\
    \label{eqn:mp2}
    \mbold^{i}_{(s)} &= \sum_{j=1}^N \mbold^{ij}_{(s)}, \\
    \label{eqn:mp3}
    \hbold^{i}_{(s)} &= \hbold^{i}_{(s-1)} + \varphi_h(\hbold^{i}_{(s-1)}, \mbold^{i}_{(s)}), \\
    \label{eqn:mp4}
    \dot{f^{i}} &= \varphi_{\dot{f}} \lp\hbold^{i}_{(\max s)} \rp \\
    \label{eqn:mp5}
    \dot{\lbold} &= \varphi_{\dot{\lbold}}\lp \frac{1}{n} \sum_{i=1}^{n} \hbold^{i}_{(\max s)} \rp
\end{align}
where $\mbold^{ij}_{(s)}, \mbold^{i}_{(s)}$ represent messages at layer $s$ between nodes $i$ and $j$, $\hbold^{j}_{(s)}$ represents hidden representation of node $j$ at layer $s$; $\varphi_m, \varphi_h, \phi_{\hbold_{(0)}}, \varphi_{\dot{f}}, \varphi_{\dot{\lbold}}$ represent parametric functions with all parameters noted together as $\theta$.  A symbol $\square \in \triangle$ with a dot above it $\dot{\square}$ represents the corresponding velocity components of the learned vector field, i.e. $\dot{\square} \coloneqq v_t^{\triangle,\theta}(\cbold_t)$. Finally, we define 
\begin{align}
    \text{SinusoidalEmbedding}(x) \coloneqq \lp \sin(2 \pi k x), \cos(2 \pi k x) \rp_{k=0, \ldots, n_{freq}}^{T},
\end{align}
where $n_{freq}$ is a hyperparameter. We standardized the $\lbold$ input to the network with z-scoring. We also standardized the outputs for predicted tangent vectors $\dot{\fbold}$, $\dot{\lbold}$. Models were trained using the \texttt{AdamW} optimizer \cite{loshchilov2018decoupled}. 

We parameterize our loss as an affine combination. That means we enforce the following condition for all experiments:
\begin{align}
    \lambda_{\lbold} + 
    \lambda_{\fbold} + 
    \lambda_{\abold} &= 1
\end{align}
enforced by
\begin{align}
    \Tilde{\lambda}_{\lbold} + \Tilde{\lambda}_{\fbold} + 
    \Tilde{\lambda}_{\abold} &\coloneqq \Tilde{\lambda}; &
    \lambda_{\lbold} &= \Tilde{\lambda}_{\lbold} / \Tilde{\lambda}, &
    \lambda_{\fbold} &= \Tilde{\lambda}_{\fbold} / \Tilde{\lambda}, &
    \lambda_{\abold} &= \Tilde{\lambda}_{\abold} / \Tilde{\lambda}.
\end{align}

In DNG, we introduce an additional loss term. When this term is included, we also include it in the affine combination.

We provide general and network hyperparameters in Table~\ref{table:general_hyperparameters} and Table~\ref{table:network_hyperparameters}. Recall, all datasets use a $60-20-20$ split between training, validation, and test data. We apply the same split as \citet{xie2021crystal} and \citet{jiao2023crystal}. More specific details exist in the corresponding experiment sections.
\begin{table}[h]
    \centering
    \begin{minipage}{0.5\linewidth}
        \centering
        \caption{General Hyperparameters}
        \label{table:general_hyperparameters}
        \resizebox{\textwidth}{!}{
        \begin{tabular}{lcccc}
        \toprule
         & \textbf{Carbon} & \textbf{Perov} & \textbf{MP-20} & \textbf{MPTS-52} \\
        \midrule
        Max Atoms & 24 & 20 & 20 & 52 \\
        Max Epochs & 8000 & 6000 & 2000 & 1000 \\
        Total Number of Samples & 10153 & 18928 & 45231 & 40476 \\
        Batch Size & 256 & 1024 & 256 & 64 \\
        \bottomrule
        \end{tabular}}
    \end{minipage}%
    \hspace{0.05\linewidth}%
    \begin{minipage}{0.4\linewidth}
        \centering
        \caption{Network Hyperparameters}
        \label{table:network_hyperparameters}
        \resizebox{0.8\textwidth}{!}{
        \begin{tabular}{lc}
        \toprule
         & \textbf{Value} \\
        \midrule
        Hidden Dimension & 512 \\
        Time Embedding Dimension & 256 \\
        Number of Layers & 6 \\
        Activation Function & silu \\
        Layer Norm & True \\
        \bottomrule
        \end{tabular}}
    \end{minipage}
\end{table}
\begin{table}[h]
    \centering
    \begin{minipage}{0.55\linewidth}
        \centering
        \caption{CSP Hyperparameters}
        \label{table:csp_hyperparameters}
        \resizebox{\textwidth}{!}{
        \begin{tabular}{lcccc}
        \toprule
         & \textbf{Carbon} & \textbf{Perov} & \textbf{MP-20} & \textbf{MPTS-52} \\
        \midrule
        Learning Rate & 0.001 & 0.0003 & 0.0001 & 0.0001 \\
        Weight Decay & 0.0 & 0.001 & 0.001 & 0.001 \\
        $\Tilde{\lambda}_{\fbold}$ (Frac Coords) & 400 & 1500 & 300 & 300 \\
        $\Tilde{\lambda}_{\lbold}$ (Lattice) & 1.0 & 1.0 & 1.0 & 1.0 \\
        \midrule
        $s'$ (Anti-Anneal Slope) & 2.0 & 1.0 & 10.0 & 5.0 \\
        Anneal $\fbold$ & False & False & True & True \\
        Anneal $\lbold$ & False & False & False & False \\
        \bottomrule
        \end{tabular}
        }
    \end{minipage}%
    \hspace{0.05\linewidth}%
    \begin{minipage}{0.35\linewidth}
        \centering
        \caption{DNG Hyperparameters}
        \label{table:dng_hyperparameters}
        \resizebox{0.8\textwidth}{!}{
        \begin{tabular}{lc}
        \toprule
         & \textbf{Value} \\
        \midrule
        Learning Rate & 0.0005 \\
        Weight Decay & 0.005 \\
        $\Tilde{\lambda}_{\abold}$ (Atom Type) & 300 \\
        $\Tilde{\lambda}_{\fbold}$ (Frac Coords) & 600 \\
        $\Tilde{\lambda}_{\lbold}$ (Lattice) & 1.0 \\
        $\Tilde{\lambda}_{\text{sce}}$ (Cross Entropy) & 20 \\
        \midrule
        $s'$ (Anti-Annealing Slope) & 5.0 \\
        Anneal $\abold$ & False \\
        Anneal $\fbold$ & True \\
        Anneal $\lbold$ & True \\
        \bottomrule
        \end{tabular}}
    \end{minipage}
\end{table}

\paragraph{Crystal Structure Prediction}
We employed the network defined above for the CSP experiments. We swept over a grid and selected the model that maximized the match rate on 2,000 reconstructions from (a subset of) the validation set.

We swept learning rate $\in \{ 0.001, 0.0003 \}$, weight decay $\in \{ 0.003, 0.001, 0.0\}$, gradient clipping = 0.5, $\Tilde{\lambda}_{\lbold} = 1$, $\Tilde{\lambda}_{\fbold} \in \{ 100, 200, 300, 400, 500 \}$.

We performed multiple reconstructions using various values for the anti-annealing velocity scheduler with coefficient $s' \in \{0, 1, 2, 3.5, 5, 10 \}$. We found that the velocity scheduler to be most effective when applied to $\fbold$ alone. Ablation tests of this phenomenon can be found in Appendix~\ref{appendix:results_continued}.

\paragraph{De novo generation}
For the unconditional experiment, we made some changes to the network above that we found favorable for the featurization of the crystal. The new network and featurization is:
\begin{align}
    \label{eqn:dng_mp0}
    \hbold^{i}_{(0)} &= \phi_{\hbold_{(0)}}(a^{i}) \\
    \label{eqn:dng_mp1}
    \mbold^{ij}_{(s)} &= \varphi_m \lp \hbold^{i}_{(s-1)}, \hbold^{j}_{(s-1)}, \lbold, \text{SinusoidalEmbedding}\lp \log_{f^{i}}(f^{j}) \rp, \zbold(n), \frac{\lboldt^{T} \lboldt \fbold}{\lVert \lboldt^{T} \lboldt \fbold \rVert}\rp, \\
    \label{eqn:dng_mp2}
    \mbold^{i}_{(s)} &= \sum_{j=1}^N \mbold^{ij}_{(s)}, \\
    \label{eqn:dng_mp3}
    \hbold^{i}_{(s)} &= \hbold^{i}_{(s-1)} + \varphi_h(\hbold^{i}_{(s-1)}, \mbold^{i}_{(s)}), \\
    \label{eqn:dng_mp4}
    \dot{f^{i}} &= \varphi_{\dot{f}} \lp\hbold^{i}_{(\max s)} \rp \\
    \label{eqn:dng_mp5}
    \dot{\lbold} &= \varphi_{\dot{\lbold}}\lp \frac{1}{n} \sum_{i=1}^{n} \hbold^{i}_{(\max s)}, \sum_{i=1}^{n} \hbold^{i}_{(\max s)} \rp \\
    \label{eqn:dng_mp6}
    \dot{a^{i}} &= \varphi_{\dot{a}} \lp\hbold^{i}_{(\max s)} \rp \\
\end{align}
where $\log_{f^{i}}(f^{j})$ is defined in \eqref{eqn:atom_wise_torus_logmap} as the logmap for the flat torus, $\zbold(n)$ represents a learned embedding of the number of atoms $n$ in the crystal's unit cell with parameters concatenated to $\theta$, $\frac{\lboldt^{T} \lboldt \fbold}{\lVert \lboldt^{T} \lboldt \fbold \rVert}$ is the cosine of the angles between the Cartesian edge between atoms and the three lattice vectors \cite{zeni2023mattergen}, $\varphi_{\dot{\lbold}}$ takes in both mean and sum pooling across nodes, and $\varphi_{\dot{a}}$ represents a parametric function with parameters concatenated to $\theta$. A symbol $\square \in \triangle$ with a dot above it $\dot{\square}$ represents the corresponding velocity components of the learned vector field, i.e. $\dot{\square} \coloneqq v_t^{\triangle,\theta}(\cbold_t)$. The additional edge features in \eqref{eqn:dng_mp1} are invariant to translation and rotation. Recall that at $t=0$, $a^{i}$ is drawn randomly from the base distribution.

We included an additional loss term with a Lagrange multiplier in our loss function. Namely a version of what \citet{chen2022analog} call \emph{sigmoid cross entropy} in Appendix B.2, adapted for atom types represented as analog bits: 
\begin{align}
    \label{eqn:sigmoid-cross-entropy}
    \quad \Lfrak_{\text{sce}} &\coloneqq - \log \sigma \lp \abold_1 \cdot \hat{\abold}_1 \rp, \\
    \text{with } \hat{\abold}_1 &\coloneqq (1-t) \, \dot{\abold}_t + \abold_t,
\end{align}
where $\sigma$ is the logistic sigmoid, $\abold_1 \in \lC -1, 1 \rC^{h}$ is the target analog bit-style atom type vector, $t$ is the time where the loss is evaluated, $\abold_t$ is the point along the path between $\abold_0$ and $\abold_1$ at time $t$, $\cdot$ represents the inner product between vectors, and $\dot{\abold}_t \coloneqq v_t^{\Acal,\theta}(\cbold_t)$.
This represents a one-step numerical estimate of the final predicted position of $\abold_t$ at $t=1$. We add this term into the objective \eqref{eqn:objective} in affine combination with unnormalized Lagrange multiplier $\Tilde{\lambda}_{\text{sce}}$ suitably normalized as $\lambda_{\text{sce}}$ .

We performed a sweep over the hyperparameters learning rate $\in \{0.0001, 0.0003, 0.0005, 0.0007, 0.001 \}$, weight decay $\in \{ 0.0, 0.0001, 0.0005, 0.001, 0.003, 0.005\}$, gradient clipping = 0.5, $\Tilde{\lambda}_{\lbold} = 1$, $\Tilde{\lambda}_{\fbold} \in \{ 40, 100, 200, 300, 400, 600, 800 \}$, $\Tilde{\lambda}_{\abold} \in \{40, 100, 200, 300, 400, 600, 800, 1200, 1600 \}$, and $\Tilde{\lambda}_{\text{sce}} \in \{ 0, 20 \}$, and velocity schedule coefficient $s' \in \{0, 1, 2, 5 \}$

We performed model selection using generated samples from each model in the sweep. After computing the proxy metrics on those samples, we collected the $\approx 50$ models that were in the top $86$th percentile on (both structural \& compositional) validity, Wasserstein distance in density ($\rho$), and Wasserstein distance in number of unique elements ($N_{el}$). From those models, we prerelaxed the generations using CHGNet \cite{deng2023chgnet} and took the model which produced the most metastable structures (CHGNet energy above hull $< 0.1$ eV/atom). We reported the results from the one which then had the best performance on Stability Rate computed using the number of metastable structures.

\section{Enforcing G-invariance of marginal probability path}
\label{appendix:marginal_probability_path_invariance}

We assume the target distribution $q$ is $G$-invariant, where $G$ is defined as in the "Symmetries of crystals" paragraph, \ie for each $g \in G$, where $g = (\sigma, Q, \tau)$ consists of (i) permutation of atoms together with their fractional coordinates, (ii) rotation of the lattice, and (iii) translation of the fractional coordinates. 
Firstly, we show that generally for marginal probability paths where $p_1 = q$ as in \eqref{eqn:unconditional_probability_path_is_invariant}, in order to have $p_t(x)$ be $G$-invariant, it is sufficient to have $p_t(x | x_1)$ satisfy a simple pairwise $G$-invariant condition.

\begin{theorem}\label{thm:cond_prob_invariant}
For pairwise $G$-invariant conditional probability path $p_t(x|x_1)$, meaning $p_t(\g x| \g x_1)=p_t(x|x_1) \;\forall g\in G,\;x,x_1\in \Ccal$, the construction in \eqref{eqn:unconditional_probability_path_is_invariant} defines a $G$-invariant marginal distribution $p_t(x)$.
\end{theorem}

\begin{proof}
\begin{align*}
p_t( \g x) 
&= \int p_t(\g x | x_1) q(x_1) d x_1 && \text{defn. from \eqref{eqn:unconditional_probability_path_is_invariant}} \\
&= \int p_t(x | \ginv x_1) q(x_1) d x_1 && \text{pairwise $G$-invariance of $p_t(x | x_1)$} \\
&= \int p_t(x | \ginv x_1) q(\ginv x_1) d x_1 && \text{$G$-invariance of $q$} \\
&= \int p_t(x | \ginv x_1) q(\ginv x_1) \underbrace{ \left| \det J_{g^{-1}} \right| }_{=1} d (\ginv x_1) && \text{change of variables}\\
&= \int p_t(x | \tilde{x}_1) q(\tilde{x}_1) d \tilde{x}_1 && \text{$\tilde{x}_1 = \ginv x_1$}\\
&= p_t(x)
\end{align*}
\end{proof}

\paragraph{Constructing conditional flows that imply pairwise $G$-invariant probability paths}

In order to construct a pairwise $G$-invariant $p_t(x | x_1)$, 
we make use of three main approaches. One is to enforce $G$-equivariant vector fields, which correspond to $G$-equivariant flows and thus generate $G$-invariant probabilities, building on the observation of \citet{kohler2020equivariant} to the pairwise case. Another is to simply make use of representations that are $G$-invariant, resulting in $G$-invariant probabilities. Finally, we take a novel approach of generalizing the construction of Riemannian Flow Matching to equivalence classes and constructing flows between equivalence classes.

For the first approach, we require a $G$-invariant base distribution and that $u_t(\g x | \g x_1) = \g u_t(x | x_1)$. This ensures the flow satisfies $\psi_t( \g x_0 | \g x_1) = \g \psi_t( x_0 | x_1)$ and thus resulting in a pairwise $G$-invariant probability path.
This property is satisfied by the use of regular geodesic paths that we use during the training of Riemannian Flow Matching, because the shortest paths connecting any $x$ and $x_1$ on the manifolds that we consider here (flat tori and Euclidean) are simply simultaneously transformed alongside $x$ and $x_1$, for transformations such as permutation and rotation. We use this approach to enforce invariance to permutation of atoms.

The second approach is to bypass the need to enforce invariance in either $u_t(x | x_1)$ or $v_t^{\theta}(x)$ by instead using a representation of that is bijective with its entire equivalence class. We use this approach to enforce invariance to rotation of the lattice, by directly modeling angles and lengths.

The third approach is enabled by a generalization of the Riemannian Flow Matching framework in the case of a $G$-invariant $q$, relaxing the assumption that the conditional probability paths $p_t(x | x_1) = \delta(x - x_1)$ at $t=1$. Instead, we allow $p_t(x | x_1) = \delta(x - \tilde{x}_1)$ as long as $\tilde{x}_1 \in [x_1]$, the equivalence class of $x_1$, \ie $[x] = \{ \g x | \g \in G \}$.

\begin{theorem}\label{thm:invariant_flow}
Allowing the possibility of $G$-invariant conditional flow $\psi_t(x_0 | x_1)$, meaning $\psi_t(x_0 | \g x_1) = \psi_t(x_0 | x_1) \;\forall g\in G,\;x,x_1\in \Ccal$, if $\psi_t(x_0 | x_1) \in [x_1] \;\forall x,x_1\in \Ccal$, then the construction of $u_t(x)$ in \eqref{eqn:unconditional_probability_path} is valid and results in a marginal distribution that satisfies $p_1 = q$. 
\end{theorem}

\begin{proof}
For general flow functions $\psi_t(x_0 | x_1)$, it follows a Dirac-delta conditional probability $p_t(x | x_0, x_1) = \delta(x - \psi_t)$, where $\psi_t$ is short-hand for $\psi_t(x_0 | x_1)$. Then we have that 
\begin{align*}
p_{t=1}(x) &= \int p_{t=1}(x | x_0, x_1) p(x_0) q(x_1) dx_0 dx_1 \\
&= \int \delta(x - \psi_{t=1}) p(x_0) q(x_1) dx_0 dx_1 \\
&= \int \delta(x - \g x_1) p(x_0) q(x_1) dx_0 dx_1 && \text{$\psi_{t=1} \in [x_1]$} \\
&= \int \delta(x - \g x_1) p(x_0) q(\g x_1) dx_0 d(\g x_1) && \text{$G$-invariance of $q$ \& change of variables} \\
&= \int \delta(x - \tilde{x}_1) p(x_0) q(\tilde{x}_1) dx_0 d \tilde{x}_1 \\
&= q(x)
\end{align*}
\end{proof}

The main implication of only needing to satisfy $\psi_1 \in [x_1]$ is that this allows us to also impose additional the constraints on our vector fields that were previously not possible. 
Specifically, we can now allow conditional vector fields that are entry-wise $G$-invariant, \ie $u_t(x | \g x_1) = u_t(x | x_1)$ and $u_t(\g x | x_1) = u_t(x | x_1)$. Note this results in flows that satisfy $\psi_t(x_0 | \g x_1) = \psi_t(x_0 | x_1)$, and importantly, this implies the flow can no longer distinguish $x_1$ from other elements in its equivalence class; the flow $\psi_t$ is purely a function of equivalence classes $[x_0]$ and $[x_1]$. However, as per above, this is still sufficient for satisfying $p_1 = q$. Simultaneously, allowing such conditional vector fields provides the means to satisfy the pairwise $G$-invariance condition we need for $G$-invariance of $p_t(x)$, \ie $p_t(\g x | \g x_1) = p_t(x | x_1)$.

\paragraph{Translation invariance with periodic boundary conditions}

On Euclidean space, one typical method of imposing translation invariance of a set of points is to remove the mean of the set and using a ``mean-free'' representation. This provides the ability to work with a representation that does not contain any information about translation, following the second approach described above. However, on flat tori (\ie with periodic boundary conditions), this approach is not possible because the mean of a set of points is not uniquely defined. Instead, we make use of the third approach described above and construct $\psi_t(f_0 | f_1)$ such that it flows to a set of fractional coordinates that is equivalent to $f_1$. Since for the flat tori, translations in the tangent plane result in translations on the manifold, we propose simply removing the translation component of the conditional vector field resulting from the geodesic construction. This results in the ``mean-free'' conditional vector field in \eqref{eqn:translation_invariant_conditional_vf}.

\end{document}